\documentclass{article}

\PassOptionsToPackage{numbers, compress}{natbib}

\usepackage[preprint]{neurips_2026}

\usepackage[utf8]{inputenc} %
\usepackage[T1]{fontenc}    %
\usepackage{hyperref}       %
\usepackage{url}            %
\usepackage{booktabs}       %
\usepackage{amsfonts}       %
\usepackage{nicefrac}       %
\usepackage{microtype}      %
\usepackage{xcolor}         %

\usepackage{amsmath}
\usepackage{amssymb}
\usepackage{mathtools}
\usepackage{amsthm}

\usepackage{algorithm}
\usepackage{algorithmic}

\usepackage{wrapfig}

\usepackage[capitalize,noabbrev]{cleveref}

\usepackage{colortbl}
\usepackage{subcaption}

\theoremstyle{plain}
\newtheorem{theorem}{Theorem}[section]
\newtheorem{proposition}[theorem]{Proposition}
\newtheorem{lemma}[theorem]{Lemma}

\theoremstyle{definition}
\newtheorem{definition}[theorem]{Definition}

\theoremstyle{remark}

\makeatletter
\newtheorem*{rep@thm}{\rep@title}
\newcommand{\newreptheorem}[2]{%
\newenvironment{rep#1}[1]{%
 \def\rep@title{#2 \ref{##1}}%
 \begin{rep@thm}}%
 {\end{rep@thm}}}
\makeatother

\usepackage[textsize=tiny]{todonotes}
\usepackage{multirow}

\title{
A Single Deep Preference-Conditioned Policy for Learning Pareto Coverage Sets
}

\author{%
  Akihiro Kubo$^{1}$ \quad Kosuke Nakanishi \quad Shin Ishii$^{1,2,3}$ \\
  $^{1}$ International Research Center for Neurointelligence, The University of Tokyo, Tokyo, Japan \\
  $^{2}$ ATR Neural Information Analysis Laboratories, Kyoto, Japan \\
  $^{3}$ Department of Information Science, Kyoto University, Kyoto, Japan \\
  Correspondence: \texttt{kubo-a@ircn.jp}
}

\begin{document}

\maketitle

\begin{abstract}
Preference-conditioned multi-objective reinforcement learning aims to learn a single policy that captures trade-offs across preferences, 
but under nonlinear scalarization the uniqueness and continuity of the preference-to-solution correspondence remain unclear.
We study this problem in tabular multi-objective Markov decision processes (MDPs) 
using smooth Tchebycheff scalarization as a monotone utility.
Under mild interior conditions on the preference set, 
we prove that each preference induces a unique Pareto-optimal return vector and that this vector depends Lipschitz-continuously on the preference, 
providing a principled foundation for preference sweeping toward dense Pareto-front coverage.
To compute these targets, we formulate the problem over occupancy measures and derive Concave Mirror Descent Policy Iteration (CMDPI), 
which achieves an $O(1/k)$ objective-suboptimality rate.
We further show that each update is equivalent to solving a Kullback-Leibler-regularized MDP with the previous policy as reference, 
yielding a policy-iteration interpretation and finite-iterate policy continuity across preferences.
We instantiate the update as a deep actor-critic algorithm preserving previous-policy regularization.
On eight MO-Gymnasium tasks, it achieves the best average hypervolume rank among recent baselines and strong expected-utility performance. 
Continuous-control experiments indicate gains beyond the discrete-action setting.
\end{abstract}

\section{Introduction}
Multi-Objective Reinforcement Learning (MORL) optimizes multiple possibly competing reward signals, so the central object is not a single optimal policy
but a set of Pareto-optimal trade-offs (Pareto set / Pareto front) \citep{white1982multiobjective,Van-Moffaert2014-bk,roijers2013survey}.
With deep reinforcement learning, MORL has been extended to high-dimensional and continuous-control domains \citep{mossalam2016modrl,pmlr-v119-xu20h}.
A prominent direction is \emph{preference-conditioned} learning, where a user-specified preference vector $\omega$ is input to a single deep model
that outputs diverse trade-off policies on demand \citep{yang2019generalized,reymond2022pareto,basaklar2022pd,pmlr-v119-xu20h,liu2025cmorl,liu2025pslmorl}.
We focus on \emph{dense coverage} of Pareto-optimal objective vectors: by scanning preferences, we aim to reach a neighborhood of any Pareto-optimal objective vector
in the objective space \citep{roijers2013survey,pirotta2015continuouspareto,Van-Moffaert2014-bk,reymond2019paretodqn,pmlr-v119-xu20h}.

A standard route is to scalarize the objective vector by a utility function and solve a single-objective problem.
In MORL taxonomies, if user utility is restricted to positive weighted sums, one targets a \emph{convex coverage set (CCS)};
if one allows arbitrary monotonically increasing utilities, one targets a \emph{Pareto coverage set (PCS)} \citep{roijers2013survey,hayes2022practical}.
CCS-oriented methods admit clean formulations (often compatible with dynamic programming / linear programming), but even for convex attainable sets,
the optimal objective vector for a fixed preference can be non-unique (e.g., a whole supporting face), which makes the target ambiguous and can induce preference-wise switching/discontinuity
\citep{boyd2004convex,roijers2013survey,Van-Moffaert2014-bk,abels2019dynamic}.
In practice, they also tend to concentrate on extreme (vertex) solutions, making dense coverage challenging \citep{roijers2013survey,Van-Moffaert2014-bk}.
PCS-oriented approaches can reach non-vertex trade-offs, but the resulting problems are typically more complex; achieving stable optimization often relies on objective modifications
(e.g., strong regularization) or heuristics, which may deviate from the Pareto front \citep{geist2019theory,lu2023multiobjective,peng2025esr}.
Figure~\ref{fig:highlights} illustrates this linear--nonlinear tension.

To bridge this gap, we adopt the Smooth Tchebycheff (STCH) scalarization \citep{Lin2024stch},
a smooth approximation \citep{beck2012smoothing} of a Tchebycheff-type nonlinear utility \citep{bowman1976relationship}.
For discounted multi-objective MDPs with finite state/action spaces, the attainable objective set $\mathcal{J}$ is a convex polytope \citep{Li2024-jv}.
For any preference
\[\omega \in \Omega^m_{++} := \{ \omega \in \mathbb{R}^m \mid \sum_l \omega_l = 1,\ \omega_l > 0\ \forall l \},\]
we prove that the STCH-induced optimal objective point $\mathbf{J}^\omega$ is uniquely defined in the objective space and is Pareto-optimal
(Propositions~\ref{lem:stch_unique} and~\ref{lem:pareto2stch}).
Moreover, the map $\omega \mapsto \mathbf{J}^\omega$ is Lipschitz continuous on
\[\Omega^m_{\delta} := \{ \omega \in \Omega^m \mid \omega_l \ge \delta\ \forall l \}\]
 for any $\delta>0$
(Theorem~\ref{thm:omega-lipsitz}),
which justifies stable preference scanning.
Consequently, $\{\mathbf{J}^\omega\}$ yields an $\varepsilon$-dense approximation of the Pareto-optimal objective set, in the sense stated in the paper.

\begin{wrapfigure}{r}{0.45\linewidth}
    \centering
    \includegraphics[width=\linewidth, trim=10 14 8 10, clip]
    {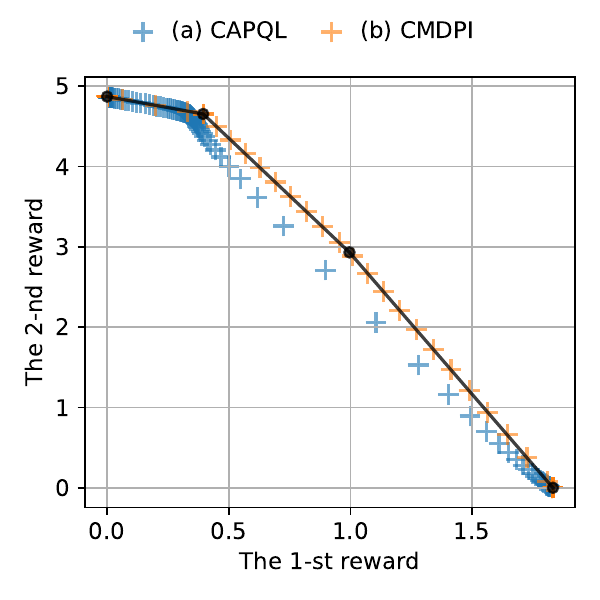}
\vspace{-2mm}
\caption{
Illustration of a key limitation of existing methods: 
linear scalarization recovers only vertex solutions of the Pareto front (black curve), 
while CAPQL exhibits biased coverage. 
CMDPI (ours) achieves denser and more uniform Pareto-optimal coverage.
}
\label{fig:highlights}
\end{wrapfigure}

To solve the resulting nonlinear scalarized problem \emph{without altering the objective}, we derive mirror descent in the occupation-measure space.
Using the Bregman divergence induced by negative conditional entropy and the relative smoothness property (Proposition~\ref{lem:om-rel_smooth}),
we obtain an $O(1/k)$ convergence guarantee under the Lu--Freund--Nesterov framework \citep{lu2017relsmooth}.
Each iteration admits an implementation via policy evaluation based on a soft Bellman fixed point and policy improvement by exponential reweighting
\citep{geist2019theory,moreno2024efficient}.
This equivalence provides an explicit iterative policy-update view of nonlinear MORL (often formulated as convex optimization over occupation measures)
\citep{Agarwal2022-lx,geist2022concave}, and makes the method suitable for deep implementations.
Importantly, at every finite iteration $k$, the update yields a policy $\pi_k^\omega$ that is continuous in $\omega$.
This stands in contrast to previous single-network MORL approaches~\citep{reymond2022pareto,yang2019generalized,basaklar2022pd}, which typically assume (rather than verify) such preference-to-policy regularity
and rely on function approximation to generalize across preferences.

Finally, although we do not provide a theoretical analysis under function approximation, 
we implement the core update in deep model-free and model-based actor--critic frameworks 
and empirically evaluate it on MO-Gymnasium~\citep{felten2023toolkit}, demonstrating improved Pareto-front approximation quality and favorable sample efficiency.
Our contributions are:
\begin{enumerate}
    \item For the STCH-scalarized rewards problem, we establish uniqueness and Pareto-optimality of $\mathbf{J}^\omega$ and Lipschitz continuity of $\omega \mapsto \mathbf{J}^\omega$,
    implying $\varepsilon$-dense coverage via preference scanning.
    \item We derive an $O(1/k)$-convergent mirror-descent algorithm in occupation-measure space that solves the nonlinear scalarized problem without objective modification.
    \item We rewrite the update in an equivalent policy-iteration form and provide a theoretical motivation for deep preference-conditioned learning,
    including $\omega$-continuity at finite iterations.
\end{enumerate}

\section{Preliminaries}
\subsection{Multi-objective discounted MDP}
We consider a discounted multi-objective Markov decision process (MOMDP) specified by a finite state space $\mathcal S$, a finite action space $\mathcal A$, a discount factor $\gamma\in(0,1)$, a transition kernel $P(\cdot\mid s,a)$, and an $m$-dimensional vector reward $\mathbf r(s,a)\in\mathbb R^m$.
The initial distribution $p_0\in\Delta(\mathcal S)$ (the probability simplex over $\mathcal S$) is assumed to be strictly positive over states, i.e., $p_0(s)>0$ for all $s\in\mathcal S$.

For a stationary randomized policy $\pi(\cdot\mid s)\in\Delta(\mathcal A)$, define
\[
\mathbf V^\pi(s)
:=\mathbb E\!\left[\sum_{t\ge 0}\gamma^t \mathbf r(s_t,a_t)\mid S_0=s,\pi\right],\quad
\mathbf J(\pi):=\mathbb E_{s\sim p_0}[\mathbf V^\pi(s)],
\]
and let $\Pi$ be the set of such policies with achievable set $\mathcal J:=\{\mathbf J(\pi):\pi\in\Pi\}$.

\subsection{Pareto optimality in MOMDP}
For $\mathbf x,\mathbf y\in\mathbb R^m$, we say that $\mathbf y$ \emph{(strictly) dominates} $\mathbf x$ if
$y_i\ge x_i$ for all $i$ and $y_j>x_j$ for some $j$.
An objective vector $\mathbf x\in\mathcal J$ is \emph{Pareto optimal} if there is no $\mathbf y\in\mathcal J$ that dominates $\mathbf x$.
A policy $\pi$ such that $\mathbf J(\pi)=\mathbf x$ for some Pareto-optimal $\mathbf x$ is called a \emph{Pareto-optimal policy}.
The set of all Pareto-optimal objective vectors $\mathcal P(\mathcal J)$ is called the \emph{Pareto front}, and a set of policies whose image in the objective space equals the Pareto front is called a \emph{Pareto coverage set (PCS)}~\citep{roijers2013survey,hayes2021practical}.
A central goal in MOMDPs is to find a PCS.

For a stationary policy $\pi$, define
\[
\mu_\pi(s,a):=(1-\gamma)\mathbb E_\pi\!\left[\sum_{t\ge0}\gamma^t \mathbf 1\{s_t=s,a_t=a\}\right],\quad
\rho_\pi(s):=\sum_a \mu_\pi(s,a),
\]
and let $\mathcal M:=\{\mu_\pi:\pi\in\Pi\}$. Then
\begin{equation}
\mathbf J(\pi)=\tfrac{1}{1-\gamma}\sum_{s,a}\mu_\pi(s,a)\mathbf r(s,a),
\label{eq:J-linear-mu}
\end{equation}
and $\mathcal M$ (hence $\mathcal J$) is a convex polytope~\citep{Li2024-jv,altman2021constrained}.

\subsection{Smooth Tchebycheff (STCH) scalarization}
We define the Smooth Tchebycheff (STCH) utility function\footnote{~\citet{Lin2024stch} study the corresponding minimization form; all statements in this paper follow from the equivalent maximization formulation obtained by flipping the signs of the objectives.}
as
\begin{equation}
    \label{def:stch}
    u(\mathbf f,\omega)
    := - \tau\log\sum_{k=1}^m \exp\!\left(\frac{\omega_k (I_k - f_k)}{\tau}\right).
\end{equation}
Here, $\omega\in\Omega^m_{++}$ is a \emph{preference vector} and $\tau>0$ is a smoothing parameter.
The vector $\mathbf I=(I_1,\ldots,I_m)^\top$ is referred to as a \emph{utopia point} in this paper, and is chosen such that each component $I_k$ upper-bounds the attainable values of $f_k$.

As shown by ~\citet{Lin2024stch}, as $\tau\to 0$, the STCH utility $u(\cdot,\omega)$ converges pointwise to the (non-smooth) weighted Tchebycheff scalarization.
Thus, $\tau$ controls the degree of smoothing via the log-sum-exp approximation.

In tabular settings, the utopia point can be chosen explicitly.
In deep or sample-based implementations, however, $\mathbf I$ is generally unknown and must be estimated from data; the estimation procedure used in this work is described in Appendix~\ref{app:impl}.

The function $u(\cdot,\omega)$ is concave, continuous, and strictly increasing component-wise, with
\begin{equation}
\nabla_{\mathbf f}u(\mathbf f,\omega)\in\mathbb R^m_{++}\qquad \forall \mathbf f\in\mathbb R^m,
\label{prop:grad_plus}
\end{equation}
where $\mathbb R^m_{++}:=\{x\in\mathbb R^m:\ x_i>0\ \forall i\}$. We consider the scalarized MOMDP
\begin{equation}
\max_{\pi\in\Pi} u(\mathbf J(\pi),\omega),
\label{eq:stch_mo_dp}
\end{equation}
which admits a unique optimal objective vector when the achievable set is polyhedral, despite $u(\cdot,\omega)$ not being strictly concave in general~\citep{Lin2024stch}.

\begin{proposition}[Uniqueness in the objective space]
\label{lem:stch_unique}
Fix $\omega\in\Omega^m_{++}$ and assume that $\mathcal J$ is a convex polytope.
Then problem~\eqref{eq:stch_mo_dp} admits a \emph{unique} optimal objective vector $\mathbf J^\omega\in\mathcal J$, and $\mathbf J^\omega$ is Pareto optimal.
\end{proposition}

\begin{proposition}
\label{lem:pareto2stch}
If $\mathbf J\in\mathcal J$ is Pareto optimal, then there exists $\omega\in\Omega^m_{++}$ such that $\mathbf J$ is an optimal objective vector of~\eqref{eq:stch_mo_dp}.
\end{proposition}

As a consequence, $\{\mathbf J^\omega:\omega\in\Omega^m_{++}\}$ coincides with the Pareto front.

Since the objective $\mathbf J^\omega$ is uniquely determined for each $\omega$, we can consider the mapping from $\omega\in\Omega^m_{++}$ to the objective space:
\begin{theorem}
\label{thm:omega-lipsitz}
For $\delta\in(0,1)$, the mapping $\phi:\Omega^m_\delta\to\mathcal J$ defined by $\phi(\omega):=\mathbf J^\omega$ is Lipschitz continuous on $\Omega^m_\delta$.
\end{theorem}
Proofs of Propositions~\ref{lem:stch_unique} and ~\ref{lem:pareto2stch}  and Theorem~\ref{thm:omega-lipsitz} are deferred to Appendix~\ref{app:stch_unique} and~\ref{app:omega-lipsitz}.

Since $\Omega^m_\delta\rightarrow \Omega^m_{++}$ as $\delta\downarrow 0$, Propositions~\ref{lem:stch_unique}--\ref{lem:pareto2stch} imply the following ``dense coverage'' interpretation:
for any Pareto-optimal objective vector $\mathbf J^\star$ and any $\varepsilon>0$, there exists $\omega\in\Omega^m_{++}$ such that
$\|\mathbf J^\omega-\mathbf J^\star\|\le\varepsilon$.
We use the term \emph{dense coverage} in this sense.
Importantly, Proposition~\ref{lem:stch_unique} allows us to describe convergence toward the unique target $\mathbf J^\omega$ even when the corresponding optimal policy is not unique.

\subsection{Entropy geometry and KL/Bregman identity}
For two policies $\pi,\pi_0$ with $\pi_0(a\mid s)>0$ for all $(s,a)$, define the occupancy-weighted conditional KL by
\begin{align}
\Gamma_\gamma(\pi,\pi_0)
:= \sum_{s,a}\mu_\pi(s,a)\log\frac{\pi(a\mid s)}{\pi_0(a\mid s)} 
= \sum_s \rho_\pi(s)\,\mathrm{KL}\!\left(\pi(\cdot\mid s)\ \middle\|\ \pi_0(\cdot\mid s)\right).
\label{eq:Gamma_def}
\end{align}

For an occupancy measure $\mu>0$ (componentwise), define the negative conditional entropy $\psi_\gamma$
and the induced Bregman divergence $D_{\psi_\gamma}$:
\[
\psi_\gamma(\mu):=\sum_{s,a}\mu(s,a)\log\frac{\mu(s,a)}{\rho_\mu(s)},
\quad
D_{\psi_\gamma}(\mu\mid\mu_0):=
\psi_\gamma(\mu)-\psi_\gamma(\mu_0)-\langle\nabla\psi_\gamma(\mu_0),\mu-\mu_0\rangle.
\]
The following lemma shows that $D_{\psi_\gamma}$ corresponds to the policy-space regularizer $\Gamma_\gamma$.

\begin{lemma}[Correspondence of $\Gamma_\gamma$ and $D_{\psi_\gamma}$ (Proof in Appendix~\ref{app:mu-breg_pi-kl})]
\label{lem:mu-breg_pi-kl}
Assume $p_0(s)>0$ for all $s$.
Let $\mu,\mu'\in\mathcal M$ correspond to policies $\pi\leftrightarrow\mu$ and $\pi'\leftrightarrow\mu'$. Then
\begin{equation}
\Gamma_\gamma(\pi,\pi')=D_{\psi_\gamma}(\mu\mid\mu')
\label{eq:mu-breg_pi-kl}
\end{equation}
\end{lemma}

\subsection{Relative smoothness and mirror descent}
\label{ss:rel-smooth}
The following notion is central to convergence analyses of mirror descent for convex optimization.
\begin{definition}[Relative smoothness]
Given a convex function $f$ and a distance-generating function $\psi$, we say that $f$ is \emph{$L$-relatively smooth with respect to $\psi$} if there exists $L>0$ such that, for all $x,y$,
\begin{equation}
f(y)\le f(x)+\langle\nabla f(x),y-x\rangle + L D_\psi(y\mid x).
\label{def:rel_smooth}
\end{equation}
\end{definition}

If $f$ is $L$-relatively smooth, mirror descent~\citep{lu2017relsmooth} with the update
\begin{equation}
x_{k+1} \in \arg\min_{x}\left\{\langle\nabla f(x_k),x\rangle + L D_\psi(x\mid x_k)\right\}
\label{eq:md_update}
\end{equation}
achieves an $O(1/k)$ rate in objective suboptimality:
\begin{equation}
f(x_k) - f^\ast \le \frac{L}{k}\, D_\psi(x^\ast\mid x_0),
\label{eq:md_gap_bound}
\end{equation}
where $f^\ast:=\min_x f(x)$ and $x^\ast\in\arg\min_x f(x)$ (not necessarily unique).
In our applications, we will instantiate $x$ with an occupancy measure and $\psi$ with $\psi_\gamma$.

\section{Mirror-Descent Policy Iteration under Fixed Preferences}
\subsection{Occupancy-Measure Formulation}

The problem in~\eqref{eq:stch_mo_dp} can be written over $\mu\in\mathcal{M}$ using~\eqref{eq:J-linear-mu} as
\[
\bar{\mathbf J}(\mu):=\tfrac{1}{1-\gamma}\sum_{s,a}\mu(s,a)\mathbf r(s,a),\quad
f(\mu):=-u(\bar{\mathbf J}(\mu),\omega).
\]
yielding the convex program
\begin{equation}
\min_{\mu\in\mathcal M} f(\mu).
\label{eq:convom-mu}
\end{equation}
Convexity follows since $u$ is concave and $\bar{\mathbf J}$ is linear. By Proposition~\ref{lem:stch_unique}, the optimal objective $\mathbf J^\omega$ is unique; hence we analyze value/objective convergence, i.e., $f(\mu_k)\to f^\star$ and $\mathbf J(\pi_k)\to \mathbf J^\omega$.

\medskip
The core idea of this work is to solve~\eqref{eq:convom-mu} by mirror descent over the discounted occupancy-measure polytope $\mathcal M$. Specifically, at iteration $k$ we perform
\begin{equation}
    \label{eq:md-mu}
    \mu_{k+1} =
    \arg\min_{\mu\in\mathcal M}
    \left\{\langle \nabla f(\mu_k),\mu\rangle + L\, D_{\psi_\gamma}(\mu\mid\mu_k)\right\}.
\end{equation}

\begin{lemma}[Occupancy-measure relative smoothness (Proof in Appendix~\ref{app:om-rel_smooth})]
    \label{lem:om-rel_smooth}
    The function $f$ is relatively smooth with respect to $\psi_\gamma$: there exists $L>0$ such that~\eqref{def:rel_smooth} holds on $\mathcal M$.
\end{lemma}

By Lemma~\ref{lem:om-rel_smooth} and the standard mirror-descent convergence result for relatively smooth objectives (Section~\ref{ss:rel-smooth}), we obtain~\eqref{eq:md_gap_bound} for any $k\ge 1$, where $x_k$ is occupation measure $\mu_k$ at iteration $k$ and $\mu^\star$ denotes occupancy measure of a Pareto optimal policy for the fixed $\omega$.

\subsection{From occupancy update to a KL-regularized policy update}
\label{ss:om2pi-kl}

The optimization~\eqref{eq:md-mu} is convex in $\mu$ and does not directly yield an implementable policy update. However, Lemma~\ref{lem:mu-breg_pi-kl} shows that
$D_{\psi_\gamma}(\mu\mid\mu_k)$ can be identified with $\Gamma_\gamma(\pi,\pi_k)$, which allows us to pull back~\eqref{eq:md-mu} to the policy space as a KL-regularized MDP~\citep{geist2019theory} with reference policy $\pi_k$.

Let $\mu_k \leftrightarrow \pi_k$. By the chain rule,
\[
\nabla_\mu f(\mu_k)(s,a)
= -\tfrac{1}{1-\gamma}\,
\langle \mathbf g_k,\mathbf r(s,a)\rangle,\quad
\mathbf g_k:=\nabla_{\mathbf J}u(\mathbf J(\pi_k),\omega)\in\mathbb R^m_{++}.
\]
Define $r_k(s,a):=\langle \mathbf g_k,\mathbf r(s,a)\rangle$. Then~\eqref{eq:md-mu} yields
\begin{equation}
\label{eq:reg-mdp}
\pi_{k+1}\in\arg\max_{\pi\in\Pi}\;
\mathbb E_\pi
\left[\sum_{t\ge0}\gamma^t r_k(s_t,a_t)\right]
- L\,\Gamma_\gamma(\pi,\pi_k).
\end{equation}

\subsection{Solving the regularized MDP: soft Bellman evaluation + multiplicative policy improvement}
\label{sss:soft_bellman_pi}
Our method can be implemented as a policy-iteration scheme that alternates between:
(i) dynamic multi-objective $\to$ single-objective weighting (via $\mathbf g_k$), and
(ii) KL-regularized greedy improvement with respect to the resulting scalar reward, as is discussed in the following.

Equation~\eqref{eq:reg-mdp} is a special case of regularized MDPs with a closed-form optimal policy~\citep{geist2019theory}. 
We set $\alpha:=L(1-\gamma)$.

\paragraph{Policy evaluation (soft Bellman fixed point)}
Given $\pi_k$ and $\alpha>0$, define
\begin{equation}
    \label{def:chi}
    \chi(s';Q,\pi_k)
    :=
    \alpha\log\sum_{b\in\mathcal A}\pi_k(b|s')\exp\!\left(\frac{Q(s',b)}{\alpha}\right).
\end{equation}
Then the $Q$-function corresponding to~\eqref{eq:reg-mdp} satisfies the soft Bellman fixed point
\begin{equation}
    \label{eq:soft-bellman}
    Q(s,a)
    =
    r_k(s,a)
    +\gamma\mathbb E_{s'\sim P(\cdot|s,a)}\!\left[\chi(s';Q,\pi_k)\right].
\end{equation}
We denote its solution by $Q_{k+1}$. This operator is a $\gamma$-contraction and thus admits a unique fixed point; see~\citet{geist2019theory}.

\paragraph{Policy improvement (KL-prox / multiplicative weights)}
Given $Q_{k+1}$, the optimizer of~\eqref{eq:reg-mdp} is obtained in closed form, state-wise:
\footnotesize
\begin{equation}
    \label{eq:pi-update}
    \pi_{k+1}(a|s)
    \propto
    \pi_k(a|s)\exp\!\left(\frac{Q_{k+1}(s,a)}{\alpha}\right).
\end{equation}
\normalsize
This is a KL-regularized greedy update with base measure $\pi_k$, and can be viewed as the policy-space representation of the mirror-descent step~\eqref{eq:md-mu}.

\medskip
Putting everything together, our proposed method is summarized as Algorithm~\ref{alg:cmdpi} in Appendix~\ref{app:exp-toy}.
Since it is derived from mirror descent applied to (the negative of) a concave utility, we refer to it as \emph{Concave Mirror Descent Policy Iteration (CMDPI)}.

\paragraph{Connection to CAPQL}
Our method is related to CAPQL~\citep{lu2023multiobjective} in the following sense:
(i) when the scalarization is linear, e.g.,
$f(\mu) = \frac{1}{1-\gamma}\left\langle \mu, \omega^\top \mathbf r(s,a)\right\rangle$,
and (ii) when the reference policy is fixed to the uniform distribution, i.e., $\pi_k = U(\Delta_{\mathcal A})$.
In this case, $\alpha$ corresponds to the coefficient of the entropy term: larger $\alpha$ imposes stronger regularization and may deviate further from optimizing the original (unregularized) objective.

\paragraph{Finite-step continuity}
Appendix~\ref{app:continuity-omega-pi} shows that exact tabular CMDPI preserves
preference-continuity of the policy for every finite iteration, assuming a
continuous full-support initialization. While the result is tabular, this
regularity motivates our preference-conditioned deep implementation.

\subsection{Convergence guarantee of the objective value}
By Lemma~\ref{lem:om-rel_smooth}, the results in Section~\ref{ss:rel-smooth}, and the definition $f(\mu)=-u(\bar{\mathbf J}(\mu),\omega)$, we obtain
\begin{equation}
    u(\mathbf J(\pi_k),\omega)\rightarrow u(\mathbf J^\omega,\omega)
\quad (k\to\infty).
\end{equation}
By compactness of $\mathcal J$ and uniqueness of the optimal objective vector,
this objective-value convergence further implies
\begin{equation}
    \mathbf J(\pi_k)\rightarrow \mathbf J^\omega
    \quad (k\to\infty).
\end{equation}

\subsection{Deep Actor-Critic Implementation Sketch}

Based on the equivalence established in Section~\ref{ss:om2pi-kl}, we propose a deep actor--critic algorithm for multi-objective RL (Algorithm~\ref{alg:pcsac}). Both the critic and actor can be naturally extended to function approximation by leveraging~\eqref{eq:soft-bellman} and~\eqref{eq:pi-update}. In particular, we adopt an off-policy actor--critic optimization style akin to Off-PAC~\citep{degris2012offpac} and, more specifically, Soft Actor--Critic (SAC)~\citep{haarnoja2018sac}.

Our goal is to build a \emph{single} function approximator that continuously represents solutions across preference vectors $\omega\in\Omega^m_\delta$.
Concretely, the critic is a neural network $\mathbf Q_\phi$ that takes $(s,a,\omega)$ as input and outputs an $m$-dimensional vector, while the actor is a neural network $\pi_\theta$ that takes $(s,\omega)$ as input and outputs a distribution over $\mathcal A$. Here $\phi$ and $\theta$ are the critic and actor parameters, respectively.

\paragraph{Critic learning}
The critic minimizes the squared Bellman residual:
\begin{equation}
    L_c(\phi)
    =
    \mathbb{E}_{s,a\sim\mathcal D,\ \omega\sim U(\cdot)}
    \left[\sum_{k=1}^m \big(y^{(k)}-Q_\phi^{(k)}(s,a,\omega)\big)^2\right],
\end{equation}
where $\mathcal D$ is a replay buffer and $U(\cdot)$ denotes the uniform distribution over $\Omega^m_\delta$.
The target component $y^{(k)}$ is the $k$-th entry of the target vector
\begin{equation}
    \mathbf y
    :=
    \mathbf r
    +\gamma \mathbf Q_{\bar\phi}(s',a',\omega)
    -\alpha \log \pi_\theta(a'|s',\omega)\,\mathbf 1,
\end{equation}
constructed from a replay sample $(s,a,\mathbf r,s')$, an action $a'\sim\pi_\theta(\cdot|s',\omega)$, and a target critic $\mathbf Q_{\bar\phi}$, which are commonly used in Deep RL field~\citep{haarnoja2018sac,Mnih2015-bw,fujimoto2018addressing}.

We justify this target as follows. As in SAC~\citep{haarnoja2018sac,geist2019theory}, assume a uniform reference policy, i.e., $\pi_k=U(\Delta_{\mathcal A})$.
Further suppose that, for any ${\bar \omega} \in \Omega^m_\delta$, $\pi_\theta$ closely approximates the policy induced by~\eqref{eq:pi-update}:
\[
    \pi^{\bar\phi}(\cdot|s,\omega)
    := {1\over Z(s)}\exp\!\left({\bar \omega}^\top\,\mathbf{Q}_{\bar\phi}(s,\cdot,\omega)/\alpha\right),
\]
where $Z(s)$ is a normalization constant.
Since $\omega^\top \mathbf 1=1$ and $\pi^{\bar\phi}(\cdot|s,\omega)$ is the maximizer for $\max_\pi \mathbb{E}_{a\sim\pi(\cdot|s)} [ q(s,\cdot) /\alpha ] - \Gamma_\gamma(\pi(\cdot|s),U)$~\citep{geist2019theory} with $q(s,\cdot) = {\bar \omega}^\top\,\mathbf{Q}_{\bar\phi}(s,\cdot,\omega)$, by analogous arguments to Section~\ref{sss:soft_bellman_pi}, we have
\[
    {\bar \omega}^\top \mathbf y
    \approx
    {\bar \omega}^\top \mathbf r
    +\gamma\,\chi\!\left(s';{\bar \omega}^\top \mathbf Q_{\bar\phi},U\right),
\]
where $\chi$ is defined in~\eqref{def:chi}. Hence $\mathbf y$ can be interpreted as an off-policy target value for the projected (scalarized) return along each ${\bar \omega}$.
With $\tilde{\mathbf{V}}_{{\bar \phi}} := \mathbf{Q}_{{\bar \phi}}(s, \tilde{a}, \omega) - \alpha \log \pi_\theta(\tilde{a}|s)\mathbf{1}$ and ${\bar \omega} = \mathbf{g}(\omega) := \nabla_{\mathbf{f}} u(\tilde{\mathbf{V}}_{{\bar \phi}}, \omega)$, we expect that $\mathbf{g}(\omega)^\top \mathbf y$ is the target for the policy evaluation~\eqref{eq:soft-bellman}.

\paragraph{Actor learning}
Equation~\eqref{eq:pi-update} is equivalent to minimizing the KL divergence to a target distribution proportional to $\pi_k\exp(Q/\alpha)$. Thus, as in SAC, for a parametric policy $\pi_\theta$ we can write
\begin{align}
    \arg\min_\theta\ 
    \mathbb E_{s\sim\mathcal D}\!\left[\mathrm{KL}\big(\pi_\theta(\cdot|s,\omega)\ \big|\ \pi^{\phi}(\cdot|s,\omega)\big)\right]
    =
    \arg\max_{\theta}\
    \mathbb{E}_{s\sim\mathcal D}
    \left[ {\bar \omega}^\top \tilde{\mathbf{V}}_{\phi}(s) \right],
\end{align}
for any ${\bar \omega}$.
Accordingly, we define the actor loss
\begin{equation}
    L_\pi(\theta)
    :=
    -\mathbb{E}_{s\sim\mathcal D}
    \left[\mathbf{g}(\omega)\!^\top \tilde{\mathbf{V}}_{\phi}(s)\right].
\end{equation}
We summarize the resulting policy optimization algorithm for deep MORL in Algorithm~\ref{alg:pcsac} (Appendix~\ref{app:impl}), 
referred to as \emph{PCSAC}.

\section{Numerical Experiments}

We empirically evaluate the proposed methods in three stages.
First, in Section~\ref{ss:exp-toy}, we study a tabular MOMDP to examine whether the proposed CMDPI update exhibits the intended planning behavior in a controlled setting.
Second, in Section~\ref{ss:exp-gym}, we evaluate discrete-action MO-Gymnasium benchmarks to test whether the proposed deep methods, PCSAC and CMDPI, scale beyond the tabular setting and improve Pareto-front approximation quality.
Finally, we evaluate continuous-control MO-Gymnasium tasks to investigate whether partially incorporating the proposed STCH-based method (PCSAC) remains effective in continuous-action domains.\footnote{
The full CMDPI update requires a KL-proximal policy optimization step derived from the mirror-descent formulation.
In continuous-action settings with expressive actor parameterizations, implementing this update stably within an off-policy actor--critic framework becomes substantially more delicate.
We therefore evaluate a partially incorporated STCH-based variant (PCSAC) in the continuous-control experiments.
}

\subsection{Tabular MOMDP (planning)}
\label{ss:exp-toy}

We study a two-objective tabular MOMDP ($|\mathcal S|=4$, $|\mathcal A|=2$) to isolate a key limitation of existing PCS methods: linear scalarization recovers only vertex solutions, while regularized methods may produce biased Pareto coverage. 
In a fully known planning setting, we compare linear scalarization, CAPQL~\citep{lu2023multiobjective}, and the proposed CMDPI. 
For $100$ preference vectors $\omega$, we compute the converged objective vectors and visualize them in Figure~\ref{fig:highlights}.

Linear scalarization recovers only vertex solutions, while CAPQL produces non-vertex solutions but exhibits a bias--coverage trade-off. 
In contrast, CMDPI consistently recovers Pareto-optimal solutions with more uniform (denser) coverage of the front. 
Full details and sweeps are provided in Appendix~\ref{app:exp-toy}.

\subsection{MO-Gymnasium}
\label{ss:exp-gym}

\paragraph{Discrete-Action Tasks}

\begin{table*}[t]
\centering
\caption{
Average ranks on eight discrete-action MO-Gymnasium tasks.
Lower is better.
The reported uncertainty is computed by stratified sampling over tasks and seeds.
CMDPI denotes the deep actor--critic implementation of the CMDPI update.
For each aggregate result, we report the average rank followed by $\pm$ the 95\% bootstrap confidence interval over eight random seeds and tasks.
Best and second-best results are highlighted in bold and gray, respectively.
}
\label{tab:discrete-rank}
\small
\begin{tabular}{lcccccc}
\toprule
 & LS & CAPQL & PreCo & C-MORL & PCSAC & CMDPI \\
\midrule
HV & $4.25 \pm 0.36$ & $4.12 \pm 0.33$ & $3.75 \pm 0.54$ & $4.12 \pm 0.69$ & \cellcolor{lightgray}$3.00 \pm 0.72$ & $\mathbf{1.75 \pm 0.33}$ \\
EUM & $3.88 \pm 0.53$ & $3.50 \pm 0.49$ & \cellcolor{lightgray} $3.38 \pm 0.57$ & $3.62 \pm 0.67$ & $3.50 \pm 0.74$ & $\mathbf{3.12 \pm 0.54}$ \\
SP & $4.75 \pm 0.35$ & $3.25 \pm 0.44$ & $4.75 \pm 0.46$ & $\mathbf{2.38 \pm 0.69}$ & \cellcolor{lightgray} $2.62 \pm 0.62$ & $3.25 \pm 0.47$ \\
\bottomrule
\end{tabular}
\end{table*}

We first evaluate eight representative discrete-action MO-Gymnasium benchmarks~\citep{felten_toolkit_2023} to investigate whether the proposed deep methods, PCSAC and CMDPI, scale beyond the tabular setting.

We compare against linear scalarization (LS), CAPQL~\citep{lu2023multiobjective}, PreCo~\citep{yangpreference}, and C-MORL~\citep{liu2025cmorl}.
PCSAC and CMDPI are our proposed STCH-based methods.
Here, CMDPI denotes a deep actor--critic implementation of the CMDPI update, where the critic target incorporates KL regularization toward the previous policy following the mirror-descent derivation.

To reduce implementation confounders, LS, CAPQL, PCSAC, and CMDPI are implemented on the same PreCo codebase.
C-MORL is included as a recent Pareto-front discovery baseline, although it is not a single preference-conditioned policy method.

All methods are trained for $2$M environment interaction steps on each task using eight random seeds.
We evaluate each trained agent by sweeping preference vectors and computing hypervolume (HV), expected utility metric (EUM)~\citep{hayes2022practical}, and sparsity (SP).
Table~\ref{tab:discrete-rank} summarizes the average ranks across the eight discrete-action tasks.
CMDPI achieves the best average rank on HV and EUM, while PCSAC also improves over LS and CAPQL in HV.
These results suggest that the STCH-based methods improves Pareto-front approximation quality, while the additional CMDPI-style regularization further improves HV/EUM performance.

CMDPI does not achieve the best SP rank; however, since SP measures geometric spacing rather than dominance or utility, we treat HV and EUM as the primary metrics.
Full task-wise results are provided in Appendix~\ref{app:exp-detail-discrete}.

\paragraph{Continuous-Action Tasks}

\begin{table*}[th]
\centering
\caption{
Average ranks on eight continuous-action MuJoCo tasks from MO-Gymnasium.
Lower is better.
Model-based methods are evaluated at $500$K real-environment steps, while model-free methods are evaluated at $2$M steps.
For each aggregate result, we report the average rank followed by $\pm$ the 95\% bootstrap confidence interval over eight random seeds and tasks.
Best results are highlighted in bold.
}
\label{tab:continuous-rank}

\resizebox{\textwidth}{!}{%
    \begin{tabular}{l cccc ccc}
    \toprule

    & \multicolumn{4}{c}{Model-Free Methods} & \multicolumn{3}{c}{Model-Based Methods} \\
    \cmidrule(lr){2-5} \cmidrule(lr){6-8}
    
    Metrics & PDMORL & COLA & CAPQL-MF & PCSAC-MF & GPI-PD & CAPQL-MB & PCSAC-MB \\
    \midrule
    HV & $3.25 \pm 0.32$ & $\mathbf{2.00 \pm 0.41}$ & $2.75 \pm 0.38$ & $\mathbf{2.00 \pm 0.33}$ & $3.00 \pm 0.01$ & $1.75 \pm 0.17$ & $\mathbf{1.25 \pm 0.17}$ \\
    EUM & $3.25 \pm 0.33$ & $\mathbf{2.00 \pm 0.40}$ & $2.62 \pm 0.36$ & $2.12 \pm 0.37$ & $3.00 \pm 0.01$ & $1.75 \pm 0.17$ & $\mathbf{1.25 \pm 0.17}$ \\
    SP & $3.62 \pm 0.18$ & $2.12 \pm 0.44$ & $\mathbf{1.75 \pm 0.25}$ & $2.50 \pm 0.32$ & $2.25 \pm 0.28$ & $2.00 \pm 0.28$ & $\mathbf{1.75 \pm 0.28}$ \\
    \bottomrule
    \end{tabular}
} %
\end{table*}

We next evaluate eight representative continuous-action MuJoCo benchmarks from MO-Gymnasium~\citep{felten_toolkit_2023} to investigate whether the proposed STCH-based update remains effective in continuous-control domains.

We compare PD-MORL~\citep{basaklar2022pd}, COLA~\citep{li2026cola}, CAPQL-MF, PCSAC-MF, GPI-PD~\citep{alegre2023sample}, CAPQL-MB, and PCSAC-MB.
Here, MF and MB denote model-free and model-based variants, respectively.
CAPQL-MF corresponds to the linear-scalarization actor--critic baseline in our implementation.

Model-based methods, namely GPI-PD, CAPQL-MB, and PCSAC-MB, are evaluated with $500$K real-environment interaction steps, while model-free methods are evaluated with $2$M steps.
This protocol follows the standard sample-efficiency motivation of model-based RL while accounting for their higher computational cost per environment step.

We evaluate each trained agent using the same HV, EUM, and SP metrics as in the discrete-action experiments.
Table~\ref{tab:continuous-rank} summarizes the average ranks on the continuous-action tasks, separately aggregated for model-free and model-based methods.

Among the model-free methods, PCSAC-MF achieves the best average HV rank, while COLA attains the strongest EUM rank.
Among the model-based methods, PCSAC-MB achieves the best average rank on both HV and EUM.

These results suggest that the proposed STCH-based methods remains effective in continuous-control domains across both model-free and model-based settings.
In particular, the model-based PCSAC-MB variant consistently improves 
over the corresponding linear-scalarization baseline (CAPQL-MB) on the primary metrics.
Full task-wise results are provided in Appendix~\ref{app:exp-detail-continuous}.

Finally, we conduct sensitivity analyses on the utopia-point scale and the relative-smoothness/KL scale parameter.
Additional details and complete sensitivity results are provided in Appendix~\ref{app:sensitivity}.

\section{Related Work}
\paragraph{Coverage sets vs.\ preference sweeping (what ``dense'' means here).}
Coverage sets---CCS and PCS---and representative computation methods are surveyed in MORL \citep{roijers2013survey,hayes2022practical}.
CCS-oriented planning typically returns a \emph{set} sufficient for any (unknown) linear weight, e.g., via linear-support style enumeration of supporting solutions \citep{rojiers2013ccs,roijers2014cg}.
This is different from our question: whether \emph{sweeping preferences} yields \emph{stable objective-space coverage}.
Even with a convex attainable set, a fixed preference can select a whole supporting face, making the preference-to-solution map set-valued and potentially causing abrupt switching under infinitesimal preference changes \citep{boyd2004convex,abels2019dynamic}.
Our work isolates this ``selection/regularity'' issue and shows that STCH admits a unique, stable objective-point selection in the tabular polytope regime.

\paragraph{Nonlinear scalarization and SER/ESR.}
Beyond linear scalarization, MORL distinguishes
\emph{scalarized expected returns} (SER: scalarize the expected return vector) and
\emph{expected scalarized returns} (ESR: take expectation of a nonlinear utility) \citep{roijers2013survey}.
ESR is closely tied to risk-sensitive/dynamic-programming viewpoints and faces time-consistency and optimization challenges \citep{howard1972risk,ruszczynski2010risk,peng2025esr}, while recent work develops practical algorithms for concave/nonlinear objectives \citep{Agarwal2022-lx,geist2022concave}.
However, for general nonlinear scalarizations it is still nontrivial whether a preference induces a \emph{unique} objective point and whether preference sweeping avoids unstable switching.
In contrast, we prove that under STCH with a polyhedral attainable set, the optimal objective vector is unique and $\omega\mapsto \mathbf J^\omega$ is Lipschitz on $\Omega_\delta^m$ (SER viewpoint), making dense objective-space coverage analyzable.

\paragraph{Deep MORL and an optimization view (single models, stability, and mirror descent).}
Classical approaches grow an explicit \emph{set of policies} to approximate coverage sets \citep{Van-Moffaert2014-bk,pirotta2015continuouspareto,reymond2019paretodqn},
whereas preference-conditioned deep methods aim to output a trade-off directly from a single model \citep{yang2019generalized,reymond2022pareto,basaklar2022pd}, including hypernetwork-style ``many policies in one'' \citep{liu2025pslmorl} and meta-policy formulations \citep{yangpreference}.
Other recent work broadens solution notions beyond linear scalarization (e.g., Pareto stationarity) \citep{lu2023multiobjective}.
Yet, scalarization-grounded guarantees that preference sweeping stably fills the objective space (via uniqueness/continuity rather than implicit generalization) remain limited.
On the optimization side, mirror descent and regularized (soft) policy-iteration schemes have been clarified for regularized MDPs \citep{geist2019theory,neu2017unified} and concave-utility RL \citep{geist2022concave,moreno2024efficient}.
Our CMDPI derives from mirror descent over occupancy measures using relative smoothness \citep{lu2017relsmooth}, yielding an $O(1/k)$ guarantee without biasing the STCH objective and providing an explicit route to $\omega$-wise regularity of finite-iteration policies.

\section{Conclusion and Future Work}
We established that, for discounted tabular MOMDPs under STCH, each preference $\omega$ induces a \emph{unique} optimal objective vector $\mathbf J^\omega$, and the mapping $\omega\mapsto \mathbf J^\omega$ is Lipschitz on $\Omega^m_\delta$.
This makes preference sweeping a principled route to dense \emph{objective-space} coverage.

We then derived CMDPI as mirror descent over discounted occupancy measures, yielding an $O(1/k)$ convergence guarantee \emph{without biasing the target objective}.
Empirically, CMDPI recovers broad Pareto-front coverage in tabular MOMDPs with improved robustness, and CMDPI-style updates improve sample efficiency on continuous-control MORL benchmarks relative to preference-conditioned baselines.

For future work, we highlight two directions.
(i) Extend the analysis to function approximation by revisiting approximate-DP error propagation through our ``soft fixed point + exponential improvement'' structure \citep{munos2003error}.
(ii) Characterize the existence/selection and $\omega$-continuity of limiting policies $\pi_\infty(\omega)$, potentially via homotopy/temperature-scheduling frameworks \citep{li2024homotopic}.

\clearpage

\bibliographystyle{unsrtnat}
\bibliography{ref}

\clearpage

\appendix

\section*{APPENDIX}
\label{secALLappendix}

\section{Uniqueness of STCH Solutions in the Objective Space}
\label{app:stch_unique}

In the main text, we treat the policy $\pi$ as the decision variable and consider
\[
  \max_{\pi\in\Pi} u(\mathbf J(\pi),\omega).
\]
Since the utility $u(\cdot,\omega)$ depends only on the objective vector, we can equivalently optimize over the
achievable objective set $\mathcal J:=\{\mathbf J(\pi):\pi\in\Pi\}\subset\mathbb R^m$:
\[
  \max_{\pi\in\Pi} u(\mathbf J(\pi),\omega)
  \;=\;
  \max_{\mathbf c\in\mathcal J} u(\mathbf c,\omega).
\]
Therefore, it suffices to establish the \emph{uniqueness of the optimal objective vector} in the objective space; we do not
require the optimal policy to be unique (multiple policies may attain the same optimal objective vector).

Below we work in a general multi-objective optimization setting, without restricting to MOMDPs.
Let $\mathcal C\subset\mathbb R^m$ be a nonempty compact convex polytope.
Fix any $I\in\mathbb R^m$ (a ``utopia'' point) that upper-bounds $\mathcal C$ componentwise.
Define
\[
  F_\omega(\mathbf c)
  := \exp\!\left(\frac{-u(\mathbf c,\omega)}{\tau}\right)
  = \sum_{i=1}^m \exp\!\left(\frac{\omega_i(I_i-c_i)}{\tau}\right).
\]
By the monotonicity of the exponential function,
\[
  \arg\max_{\mathbf c\in\mathcal C} u(\mathbf c,\omega)
  \;=\;
  \arg\min_{\mathbf c\in\mathcal C} F_\omega(\mathbf c).
\]

\begin{proposition}[Uniqueness of STCH solutions and Pareto optimality in the objective space]
\label{lem:stch_unique_obj}
Fix $\omega\in\mathbb R^m_{++}$ and let $\mathcal C\subset\mathbb R^m$ be a nonempty compact convex polytope.
Then the problem
\[
  \max_{\mathbf c\in\mathcal C} u(\mathbf c,\omega)
\]
admits a \emph{unique} optimal solution $\mathbf c^\omega\in\mathcal C$.
Moreover, $\mathbf c^\omega$ is Pareto optimal (for maximization).
\end{proposition}

\begin{proof}
\emph{(Existence.)} Since $\mathcal C$ is compact and $u(\cdot,\omega)$ is continuous, an optimum exists.

\emph{(Pareto optimality.)} Because $\omega\in\mathbb R^m_{++}$, the utility $u(\cdot,\omega)$ is strictly increasing
in each component. Hence, if $\mathbf c'\ge \mathbf c$ and $\mathbf c'\neq \mathbf c$, then
$u(\mathbf c',\omega)>u(\mathbf c,\omega)$. Therefore, the optimizer $\mathbf c^\omega$ cannot be dominated by any other feasible point,
and is Pareto optimal. This is also a special case of the general fact proved in \citet{Lin2024stch} that STCH optima with strictly
positive weights are Pareto optimal.

\emph{(Uniqueness.)} Since the maximization is equivalent to $\min_{\mathbf c\in\mathcal C}F_\omega(\mathbf c)$,
it suffices to show that $F_\omega$ is strictly convex.
Indeed,
\[
  F_\omega(\mathbf c)
  =\sum_{i=1}^m \exp\!\left(\frac{\omega_i I_i}{\tau}\right)\exp\!\left(-\frac{\omega_i}{\tau}c_i\right),
\]
so its Hessian is diagonal:
\[
  \nabla^2 F_\omega(\mathbf c)
  = \mathrm{diag}\!\left(
     \left(\frac{\omega_1}{\tau}\right)^2 e^{\omega_1(I_1-c_1)/\tau},\;
     \ldots,\;
     \left(\frac{\omega_m}{\tau}\right)^2 e^{\omega_m(I_m-c_m)/\tau}
  \right).
\]
Each diagonal entry is strictly positive because $\omega_i>0$ and the exponential term is positive; hence
$\nabla^2 F_\omega(\mathbf c)\succ 0$ for all $\mathbf c$, which implies that $F_\omega$ is (globally) strictly convex.
Therefore the minimizer over the convex set $\mathcal C$ is unique, and so is the maximizer of $u$.
\end{proof}

\medskip

Next we show that every Pareto optimal point of a polytope is optimal for some STCH preference.
We first recall the definition of the normal cone. For a convex set $\mathcal C$ and a point $\mathbf c^\star\in\mathcal C$, define
\[
  N_{\mathcal C}(\mathbf c^\star)
  := \left\{\mathbf w\in\mathbb R^m:
    \langle \mathbf w,\mathbf c-\mathbf c^\star\rangle \le 0 \ \ \forall \mathbf c\in\mathcal C
  \right\}.
\]

\begin{lemma}[For polytopes, Pareto optimality implies the existence of a strictly positive normal]
\label{lem:posnormal_from_pareto_polytope}
Let $\mathcal C\subset\mathbb R^m$ be a nonempty compact convex polytope, and let $\mathbf c^\star\in\mathcal C$ be Pareto optimal
(for maximization), i.e., there is no $\mathbf c\in\mathcal C$ such that $\mathbf c\ge \mathbf c^\star$ and $\mathbf c\neq \mathbf c^\star$.
Then
\[
  N_{\mathcal C}(\mathbf c^\star)\cap \mathbb R^m_{++}\neq\varnothing.
\]
\end{lemma}

\begin{proof}
We argue by contradiction. Assume $N_{\mathcal C}(\mathbf c^\star)\cap \mathbb R^m_{++}=\varnothing$.
Since $\mathcal C$ is a polytope, there exist a matrix $A$ and vector $b$ such that $\mathcal C=\{\mathbf c: A\mathbf c\le b\}$.
Let
\[
  \mathcal I := \{j:\ a_j^\top \mathbf c^\star = b_j\}
\]
be the set of active constraints at $\mathbf c^\star$ (where $a_j^\top$ denotes the $j$-th row of $A$).
A standard polyhedral fact is that
\[
  N_{\mathcal C}(\mathbf c^\star)=\mathrm{cone}\{a_j:\ j\in\mathcal I\},
  \qquad
  T_{\mathcal C}(\mathbf c^\star)=\{\mathbf d:\ a_j^\top \mathbf d\le 0\ \forall j\in\mathcal I\},
\]
where $T_{\mathcal C}$ denotes the tangent cone.

Under the assumption that the cone $\mathrm{cone}\{a_j:j\in\mathcal I\}$ does not intersect $\mathbb R^m_{++}$,
a separation argument for convex cones yields the existence of
$\mathbf d\in\mathbb R^m_+\setminus\{\mathbf 0\}$ such that $a_j^\top \mathbf d\le 0$ for all $j\in\mathcal I$;
that is,
\[
  \mathbf d\in T_{\mathcal C}(\mathbf c^\star)\cap\big(\mathbb R^m_+\setminus\{\mathbf 0\}\big).
\]
For polytopes, if $\mathbf d\in T_{\mathcal C}(\mathbf c^\star)$ then there exists $\varepsilon>0$ small enough such that
$\mathbf c^\star+\varepsilon\mathbf d\in\mathcal C$.
Since $\mathbf d\in\mathbb R^m_+\setminus\{\mathbf 0\}$, we have $\mathbf c^\star+\varepsilon\mathbf d\ge \mathbf c^\star$ and
$\mathbf c^\star+\varepsilon\mathbf d\neq \mathbf c^\star$, contradicting the Pareto optimality of $\mathbf c^\star$.
Hence the assumption is false, and the claim follows.
\end{proof}

\medskip

\begin{proposition}[Every Pareto optimal point is STCH-optimal for some strictly positive preference]
\label{lem:pareto2stch_obj}
Let $\mathcal C\subset\mathbb R^m$ be a nonempty compact convex polytope.
If $\mathbf c^\star\in\mathcal C$ is Pareto optimal (for maximization), then there exists $\omega\in\mathbb R^m_{++}$ such that
$\mathbf c^\star$ is an optimizer of
\[
  \max_{\mathbf c\in\mathcal C} u(\mathbf c,\omega).
\]
\end{proposition}

\begin{proof}
For convenience we rewrite the problem as a minimization problem via a translation.
Let $\mathbf f := I-\mathbf c$, $\mathcal F:=\{I-\mathbf c:\mathbf c\in\mathcal C\}$, and $\mathbf f^\star:=I-\mathbf c^\star$.
Then $\mathcal F$ is also a nonempty compact convex polytope, and Pareto optimality of $\mathbf c^\star$ is equivalent to
Pareto minimality of $\mathbf f^\star$ (componentwise minimization). Moreover,
\[
  \max_{\mathbf c\in\mathcal C} u(\mathbf c,\omega)
  \quad\Longleftrightarrow\quad
  \min_{\mathbf f\in\mathcal F} \; G_\omega(\mathbf f),
\qquad
  G_\omega(\mathbf f):=\sum_{i=1}^m \exp\!\left(\frac{\omega_i f_i}{\tau}\right),
\]
since $u(\mathbf c,\omega)=-\tau\log G_\omega(I-\mathbf c)$ is a monotone transformation.

By Lemma~\ref{lem:posnormal_from_pareto_polytope} applied to $\mathcal C$ and $\mathbf c^\star$, there exists
$\lambda\in N_{\mathcal C}(\mathbf c^\star)\cap\mathbb R^m_{++}$.
Under the affine change of variables $\mathbf f=I-\mathbf c$, the corresponding normal direction at
$\mathbf f^\star$ is $-\lambda$, that is,
\[
  -\lambda\in N_{\mathcal F}(\mathbf f^\star).
\]
Let $a_i:=f_i^\star/\tau\ (\ge 0)$ and define for each $i$
\[
  \phi_i(t):=t\,e^{a_i t}\qquad(t>0).
\]
Since $a_i\ge 0$, $\phi_i$ is continuous, strictly increasing, and bijective from $(0,\infty)$ to $(0,\infty)$.
Hence, for any constant $c>0$, we can define
\[
  \omega_i := \phi_i^{-1}(c\,\lambda_i)\in(0,\infty).
\]

For this $\omega$, the gradient satisfies
\[
  \frac{\partial}{\partial f_i}G_\omega(\mathbf f)
  =\frac{\omega_i}{\tau}\exp\!\left(\frac{\omega_i f_i}{\tau}\right),
\]
and therefore at $\mathbf f^\star$,
\[
  \big(\nabla G_\omega(\mathbf f^\star)\big)_i
  =\frac{1}{\tau}\,\omega_i e^{a_i\omega_i}
  =\frac{1}{\tau}\,\phi_i(\omega_i)
  =\frac{c}{\tau}\,\lambda_i.
\]
Thus $\nabla G_\omega(\mathbf f^\star) = \alpha\,\lambda$ with $\alpha=c/\tau>0$. Therefore,
\[
  -\nabla G_\omega(\mathbf f^\star)
  =
  \alpha(-\lambda)
  \in N_{\mathcal F}(\mathbf f^\star),
\]
because $-\lambda\in N_{\mathcal F}(\mathbf f^\star)$ and the normal cone is a cone.
Equivalently,
\[
  \left\langle
    \nabla G_\omega(\mathbf f^\star),
    \mathbf f-\mathbf f^\star
  \right\rangle
  \ge 0
  \qquad
  \forall \mathbf f\in\mathcal F.
\]
Because $G_\omega$ is convex and $C^1$, this is the first-order optimality condition for convex minimization,
implying $\mathbf f^\star\in\arg\min_{\mathbf f\in\mathcal F} G_\omega(\mathbf f)$.
Returning to the original variables gives $\mathbf c^\star\in\arg\max_{\mathbf c\in\mathcal C}u(\mathbf c,\omega)$.
\end{proof}

\section{Lipschitz continuity of the STCH optimizer map}
\label{app:omega-lipsitz}

\begin{lemma}[Lipschitz continuity of the STCH optimizer map]
\label{lem:omega-lipschitz-general-en}
Let $\mathcal{C}\subset\mathbb{R}^m$ be a nonempty compact convex set (in particular, it may be a convex polytope).
Fix $\tau>0$, and assume that a reference point $\mathbf I\in\mathbb{R}^m$ strictly dominates $\mathcal{C}$ coordinatewise:
\[
I_i>\max_{\mathbf c\in\mathcal C} c_i\qquad (i=1,\dots,m).
\]
For any $\delta\in(0,1)$, define
\[
\Omega_\delta^m:=\Bigl\{\omega\in\mathbb{R}^m \,\Big|\, \sum_{i=1}^m\omega_i=1,\ \omega_i\ge\delta\ (i=1,\dots,m)\Bigr\}.
\]
Consider the smooth Tchebycheff utility
\[
u(\mathbf c,\omega):=-\tau\log\sum_{i=1}^m \exp\!\Bigl(\tfrac{\omega_i(I_i-c_i)}{\tau}\Bigr)
\quad (\mathbf c\in\mathcal C,\ \omega\in\Omega_\delta^m).
\]
For each $\omega\in\Omega_\delta^m$, let
\[
\mathbf c^\omega\in\arg\max_{\mathbf c\in\mathcal C} u(\mathbf c,\omega)
\]
and assume that this maximizer is unique (e.g., this is guaranteed by Propsition~\ref{lem:stch_unique_obj} in the main text).
Then the map $\phi:\Omega_\delta^m\to\mathcal C$ defined by $\phi(\omega):=\mathbf c^\omega$ is Lipschitz continuous on $\Omega_\delta^m$.
More precisely, there exists a constant $L<\infty$ such that for all $\omega,\omega'\in\Omega_\delta^m$,
\[
\|\mathbf c^\omega-\mathbf c^{\omega'}\|\ \le\ L\,\|\omega-\omega'\|,
\]
where $\|\cdot\|$ denotes any norm equivalent to the Euclidean norm.
\end{lemma}

\begin{proof}
Since $-\tau\log(\cdot)$ is strictly decreasing, we have
\[
\arg\max_{\mathbf c\in\mathcal C}u(\mathbf c,\omega)
=\arg\min_{\mathbf c\in\mathcal C}\Phi(\mathbf c,\omega),
\qquad
\Phi(\mathbf c,\omega):=\sum_{i=1}^m \exp\!\Bigl(\tfrac{\omega_i(I_i-c_i)}{\tau}\Bigr).
\]
Thus, we treat $\mathbf c^\omega$ as the unique minimizer of $\Phi(\cdot,\omega)$.

By compactness of $\mathcal C$, define
\[
M_i^+:=\max_{\mathbf c\in\mathcal C} c_i,\qquad M_i^-:=\min_{\mathbf c\in\mathcal C} c_i,
\]
and set $\alpha_i:=I_i-M_i^+>0$, $\beta_i:=I_i-M_i^-<\infty$, and $\alpha_*:=\min_i\alpha_i$.
Then for any $\mathbf c\in\mathcal C$, we have $I_i-c_i\in[\alpha_i,\beta_i]$.

First, the Hessian of $\Phi(\cdot,\omega)$ with respect to $\mathbf c$ is diagonal:
\[
\nabla^2_{\mathbf c\mathbf c}\Phi(\mathbf c,\omega)
=\mathrm{diag}\Bigl(\Bigl(\tfrac{\omega_i}{\tau}\Bigr)^2\exp\!\Bigl(\tfrac{\omega_i(I_i-c_i)}{\tau}\Bigr)\Bigr)_{i=1}^m.
\]
For $\omega\in\Omega_\delta^m$, we have $\omega_i\ge\delta$ and $I_i-c_i\ge\alpha_*$, hence
\[
\Bigl(\tfrac{\omega_i}{\tau}\Bigr)^2\exp\!\Bigl(\tfrac{\omega_i(I_i-c_i)}{\tau}\Bigr)
\ \ge\
\Bigl(\tfrac{\delta}{\tau}\Bigr)^2\exp\!\Bigl(\tfrac{\delta\,\alpha_*}{\tau}\Bigr)
=: \mu\ >0.
\]
Therefore, for every $\omega\in\Omega_\delta^m$, the function $\Phi(\cdot,\omega)$ is $\mu$-strongly convex on $\mathcal C$, uniformly in $\omega$.
Equivalently, its gradient is $\mu$-strongly monotone (with respect to the same norm):
\begin{equation}
\label{eq:strongmono-en}
\big\langle \nabla_{\mathbf c}\Phi(\mathbf x,\omega)-\nabla_{\mathbf c}\Phi(\mathbf y,\omega),\ \mathbf x-\mathbf y\big\rangle
\ \ge\ \mu\|\mathbf x-\mathbf y\|^2
\qquad (\mathbf x,\mathbf y\in\mathcal C).
\end{equation}

Next, we bound the Lipschitz constant of $\nabla_{\mathbf c}\Phi$ with respect to $\omega$.
The gradient is
\[
\nabla_{\mathbf c}\Phi(\mathbf c,\omega)
=\Bigl(-\tfrac{\omega_i}{\tau}\exp\!\bigl(\tfrac{\omega_i(I_i-c_i)}{\tau}\bigr)\Bigr)_{i=1}^m,
\]
and the Jacobian $\nabla_\omega\nabla_{\mathbf c}\Phi(\mathbf c,\omega)$ is also diagonal. For each coordinate $i$,
\[
\frac{\partial}{\partial \omega_i}\Bigl(\frac{\partial \Phi}{\partial c_i}\Bigr)
= -\frac{1}{\tau}\Bigl(1+\tfrac{\omega_i(I_i-c_i)}{\tau}\Bigr)\exp\!\Bigl(\tfrac{\omega_i(I_i-c_i)}{\tau}\Bigr).
\]
Since $\omega_i\le 1$ and $I_i-c_i\le\beta_i$, we have $\tfrac{\omega_i(I_i-c_i)}{\tau}\le \tfrac{\beta_i}{\tau}$, and thus for all $(\mathbf c,\omega)\in\mathcal C\times\Omega_\delta^m$,
\[
\big\|\nabla_\omega\nabla_{\mathbf c}\Phi(\mathbf c,\omega)\big\|
\ \le\ 
\max_{i}\ \frac{1}{\tau}\Bigl(1+\tfrac{\beta_i}{\tau}\Bigr)\exp\!\Bigl(\tfrac{\beta_i}{\tau}\Bigr)
=:U<\infty.
\]
(For a diagonal matrix, the operator norm is bounded by the maximum absolute diagonal entry.)
Hence, by the mean value inequality, for any $\mathbf c\in\mathcal C$ and any $\omega,\omega'\in\Omega_\delta^m$,
\begin{equation}
\label{eq:gradlip-en}
\big\|\nabla_{\mathbf c}\Phi(\mathbf c,\omega)-\nabla_{\mathbf c}\Phi(\mathbf c,\omega')\big\|
\ \le\ U\,\|\omega-\omega'\|.
\end{equation}

We now use first-order optimality conditions written with the normal cone $N_{\mathcal C}(\cdot)$:
\[
\mathbf 0\in \nabla_{\mathbf c}\Phi(\mathbf c^\omega,\omega)+N_{\mathcal C}(\mathbf c^\omega),
\qquad
\mathbf 0\in \nabla_{\mathbf c}\Phi(\mathbf c^{\omega'},\omega')+N_{\mathcal C}(\mathbf c^{\omega'}).
\]
Thus, there exist $\mathbf v\in N_{\mathcal C}(\mathbf c^\omega)$ and $\mathbf v'\in N_{\mathcal C}(\mathbf c^{\omega'})$ such that
\[
\nabla_{\mathbf c}\Phi(\mathbf c^\omega,\omega)+\mathbf v=\mathbf 0,\qquad
\nabla_{\mathbf c}\Phi(\mathbf c^{\omega'},\omega')+\mathbf v'=\mathbf 0.
\]
Let $\mathbf d:=\mathbf c^\omega-\mathbf c^{\omega'}$. Taking the inner product of the difference of the two equations with $\mathbf d$ yields
\[
\big\langle \nabla_{\mathbf c}\Phi(\mathbf c^\omega,\omega)-\nabla_{\mathbf c}\Phi(\mathbf c^{\omega'},\omega'),\ \mathbf d\big\rangle
+\big\langle \mathbf v-\mathbf v',\ \mathbf d\big\rangle=0.
\]
By monotonicity of the normal cone, we have $\langle \mathbf v-\mathbf v',\mathbf d\rangle\ge 0$, hence
\[
\big\langle \nabla_{\mathbf c}\Phi(\mathbf c^\omega,\omega)-\nabla_{\mathbf c}\Phi(\mathbf c^{\omega'},\omega'),\ \mathbf d\big\rangle\le 0.
\]
Decomposing the left-hand side using the same $\omega$ gives
\begin{align*}
&\big\langle \nabla_{\mathbf c}\Phi(\mathbf c^\omega,\omega)-\nabla_{\mathbf c}\Phi(\mathbf c^{\omega'},\omega),\ \mathbf d\big\rangle
+\big\langle \nabla_{\mathbf c}\Phi(\mathbf c^{\omega'},\omega)-\nabla_{\mathbf c}\Phi(\mathbf c^{\omega'},\omega'),\ \mathbf d\big\rangle
\ \le\ 0.
\end{align*}
The first term is lower-bounded by $\mu\|\mathbf d\|^2$ via \eqref{eq:strongmono-en}.
The second term is bounded by Cauchy--Schwarz and \eqref{eq:gradlip-en}:
\[
\big\langle \nabla_{\mathbf c}\Phi(\mathbf c^{\omega'},\omega)-\nabla_{\mathbf c}\Phi(\mathbf c^{\omega'},\omega'),\ \mathbf d\big\rangle
\ \ge\ -\big\|\nabla_{\mathbf c}\Phi(\mathbf c^{\omega'},\omega)-\nabla_{\mathbf c}\Phi(\mathbf c^{\omega'},\omega')\big\|\,\|\mathbf d\|
\ \ge\ -U\|\omega-\omega'\|\,\|\mathbf d\|.
\]
Combining these inequalities yields
\[
\mu\|\mathbf d\|^2\ \le\ U\|\omega-\omega'\|\,\|\mathbf d\|.
\]
If $\mathbf d=\mathbf 0$ the claim is trivial; otherwise, dividing both sides by $\|\mathbf d\|$ gives
\[
\|\mathbf c^\omega-\mathbf c^{\omega'}\|=\|\mathbf d\|\ \le\ \frac{U}{\mu}\,\|\omega-\omega'\|.
\]
Therefore $\phi$ is Lipschitz continuous on $\Omega_\delta^m$ with constant $L:=U/\mu$.
\end{proof}

\section{Equivalence between the Bregman divergence and the discounted-occupancy-weighted conditional KL}
\label{app:mu-breg_pi-kl}

Define
\[
\mathrm{ri}(\mathcal M)
:={\mu\in\mathcal M:\ \mu(s,a)>0\ \forall (s,a)}.
\]
Assume that $p_0(s)>0$ for all $s$, and that there exists a full-support policy $\bar\pi$ (i.e., $\bar\pi(a\mid s)>0$ for all $(s,a)$). Then
\[
\mu^{\bar\pi}(s,a)\ \ge\ (1-\gamma)p_0(s)\bar\pi(a\mid s)\ >0,
\]
which implies $\mu^{\bar\pi}\in \mathrm{ri}(\mathcal M)$ and hence $\mathrm{ri}(\mathcal M)\neq\varnothing$.

If $\mu(s,a)>0$, then by the previous rewriting we can differentiate $\psi_\gamma$ directly, obtaining
\[
\frac{\partial \psi_\gamma(\mu)}{\partial \mu(s,a)}
=\log\frac{\mu(s,a)}{\rho_\mu(s)}
=\log \pi_\mu(a\mid s).
\]
Therefore,
\begin{equation}
    \label{eq:grad}
(\nabla\psi_\gamma(\mu))(s,a)=\log \pi_\mu(a\mid s)
\quad (\mu\in\mathrm{ri}(\mathcal M)).    
\end{equation}

Fix $\mu_0\in\mathrm{ri}(\mathcal M)$ and take an arbitrary $\mu\in\mathcal M$.
Using \eqref{eq:grad} and $\psi_\gamma(\mu)=\sum_{s,a}\mu(s,a)\log\pi_\mu(a\mid s)$, we have
\[
\begin{aligned}
D_{\psi_\gamma}(\mu|\mu_0)
&=\sum_{s,a}\mu(s,a)\log\pi_\mu(a|s)-\sum_{s,a}\mu_0(s,a)\log\pi_{\mu_0}(a|s)
\quad-\sum_{s,a}\log\pi_{\mu_0}(a|s)\big(\mu(s,a)-\mu_0(s,a)\big)\\
&=\sum_{s,a}\mu(s,a)\log\frac{\pi_\mu(a|s)}{\pi_{\mu_0}(a|s)}\\
&=\Gamma_\gamma(\mu,\mu_0),
\end{aligned}
\]
where the last line simply rewrites $\Gamma_\gamma(\mu,\mu_0)$ in terms of the policies induced by the occupancy measures, namely $\pi_\mu$ and $\pi_{\mu_0}$.

Hence,
\[
D_{\psi_\gamma}(\mu|\mu_0)=\Gamma_\gamma(\mu,\mu_0)
\quad
(\mu_0\in\mathrm{ri}(\mathcal M),\ \mu\in\mathcal M).
\]

\section{Proof of Lemma~\ref{lem:om-rel_smooth}}
\label{app:om-rel_smooth}

This appendix makes explicit the definitions of $f$ and the reference function $\psi_\gamma$ used in Lemma~\ref{lem:om-rel_smooth},
and clarifies how the (discounted) occupancy-measure theorem induces a unique correspondence from an occupancy measure to a stationary policy.
We then prove a concrete form of relative smoothness on the feasible set of occupancy measures.

\begin{proposition}[Relative smoothness on the occupancy-measure space (concrete form)]
\label{prop:om-rel_smooth}
Fix finite sets $\mathcal S,\mathcal A$ and a discount factor $\gamma\in(0,1)$.
Let $\mathcal M\subset\mathbb{R}^{|\mathcal S||\mathcal A|}$ be a feasible set of (discounted) occupancy measures.
Assume that $\mathcal M$ is nonempty, is a convex polytope, and has a nonempty relative interior $\mathrm{relint}(\mathcal M)$.
For each $\mu\in\mathrm{relint}(\mathcal M)$, define
\[
\rho_\mu(s):=\sum_{a\in\mathcal A}\mu(s,a),
\qquad
\pi_\mu(a|s):=\frac{\mu(s,a)}{\rho_\mu(s)}
\]
(which is well-defined since $\rho_\mu(s)>0$ on $\mathrm{relint}(\mathcal M)$).

Let $\mathbf r(s,a)\in\mathbb{R}^m$ be componentwise bounded and set
$R_{1,\max}:=\max_{s,a}\|\mathbf r(s,a)\|_1<\infty$.
Define the linear map
\[
\mathbf J(\mu):=\frac{1}{1-\gamma}\sum_{s,a}\mu(s,a)\mathbf r(s,a)\in\mathbb{R}^m.
\]
Fix a temperature $\tau>0$ and a preference vector $\omega\in\mathbb{R}^m_{++}$, and let $u(\cdot,\omega)$ be the Smooth Tchebycheff utility.
Define
\[
g(\mathbf z):=-u(\mathbf z,\omega),
\qquad
f(\mu):=g(\mathbf J(\mu)).
\]
Moreover, define the reference function (negative conditional entropy)
\[
\psi_\gamma(\mu):=\sum_{s,a}\mu(s,a)\log\frac{\mu(s,a)}{\rho_\mu(s)}
\qquad(\mu\in\mathrm{relint}(\mathcal M)),
\]
and its Bregman divergence $D_{\psi_\gamma}(\nu|\mu)$.

Then there exists a constant $L_{\mathrm{rel}}>0$ such that, for all
$\mu,\nu\in\mathrm{relint}(\mathcal M)$,
\[
f(\nu)\le
f(\mu)+\langle\nabla f(\mu),\nu-\mu\rangle
+L_{\mathrm{rel}}\,D_{\psi_\gamma}(\nu|\mu).
\]
In particular, one may take
\[
L_{\mathrm{rel}}
=\frac{\omega_{\max}^2 R_{1,\max}^2}{4\tau(1-\gamma)^4},
\qquad \omega_{\max}:=\max_{k\in[m]}\omega_k.
\]
Hence $f$ is relatively smooth with respect to $\psi_\gamma$.
\end{proposition}

\begin{proof}
\textbf{(Bridging the notation: uniqueness of $\mu\mapsto \pi_\mu$).}
If $\mu\in\mathrm{relint}(\mathcal M)$, then $\mu(s,a)>0$ holds, and hence
$\rho_\mu(s)=\sum_a\mu(s,a)>0$ so that $\pi_\mu(\cdot|s)$ is a well-defined probability distribution.
Moreover, since $\mu(s,a)=\rho_\mu(s)\pi_\mu(a|s)$, the conditional distribution $\pi_\mu$ induced by a given $\mu$ is unique.
Therefore $f(\mu)=g(\mathbf J(\mu))$ is uniquely defined as a function of $\mu$.

\medskip
We proceed in two steps:
(1) $\ell_1$-smoothness of $f$ (Lipschitz continuity of the gradient in the dual norm $\ell_\infty$),
and (2) an inequality controlling $\|\nu-\mu\|_1^2$ by $D_{\psi_\gamma}(\nu|\mu)$.
Combining them yields the desired relative smoothness inequality.

\medskip
\textbf{(1) $\ell_1$-smoothness: $\|\nabla f(\nu)-\nabla f(\mu)\|_\infty\le L_1\|\nu-\mu\|_1$.}
Since $g(\mathbf z)=\tau\log\sum_{k=1}^m\exp(\omega_k(I_k-z_k)/\tau)$ is a log-sum-exp function,
its Hessian can be written as
\[
\nabla^2 g(\mathbf z)=\frac{1}{\tau}\Omega\bigl(\mathrm{Diag}(p)-pp^\top\bigr)\Omega,
\]
where $p$ is the softmax vector and $\Omega=\mathrm{diag}(\omega)$.
For each $k,l$, we have $|(\mathrm{Diag}(p)-pp^\top)_{kl}|\le 1/4$, hence
\[
\bigl|(\nabla^2 g(\mathbf z))_{kl}\bigr|
\le \frac{\omega_{\max}^2}{4\tau}.
\]
By the mean value theorem, for any $\Delta\mathbf z$ and any component $k$,
\[
\bigl|(\nabla g(\mathbf z+\Delta\mathbf z)-\nabla g(\mathbf z))_k\bigr|
\le \sum_{l=1}^m \bigl|(\nabla^2 g(\tilde{\mathbf z}))_{kl}\bigr|\,|\Delta z_l|
\le \frac{\omega_{\max}^2}{4\tau}\|\Delta\mathbf z\|_1,
\]
that is,
\begin{equation}
\label{eq:g-grad-lip}
\|\nabla g(\mathbf z+\Delta\mathbf z)-\nabla g(\mathbf z)\|_\infty
\le \frac{\omega_{\max}^2}{4\tau}\|\Delta\mathbf z\|_1.
\end{equation}

Since $\mathbf J(\mu)$ is linear, the chain rule gives, for each $(s,a)$,
\[
\frac{\partial f}{\partial \mu(s,a)}(\mu)
=\left\langle \nabla g(\mathbf J(\mu)), \frac{1}{1-\gamma}\mathbf r(s,a)\right\rangle.
\]
Thus,
\[
\left|\frac{\partial f}{\partial \mu(s,a)}(\nu)-\frac{\partial f}{\partial \mu(s,a)}(\mu)\right|
\le \frac{\|\mathbf r(s,a)\|_1}{1-\gamma}
\bigl\|\nabla g(\mathbf J(\nu))-\nabla g(\mathbf J(\mu))\bigr\|_\infty.
\]
Using $\|\mathbf r(s,a)\|_1\le R_{1,\max}$ and \eqref{eq:g-grad-lip}, we obtain
\[
\|\nabla f(\nu)-\nabla f(\mu)\|_\infty
\le \frac{R_{1,\max}}{1-\gamma}\cdot\frac{\omega_{\max}^2}{4\tau}
\|\mathbf J(\nu)-\mathbf J(\mu)\|_1.
\]
Moreover, by the definition of $\mathbf J$ and the triangle inequality,
\[
\|\mathbf J(\nu)-\mathbf J(\mu)\|_1
=\left\|\frac{1}{1-\gamma}\sum_{s,a}(\nu-\mu)(s,a)\mathbf r(s,a)\right\|_1
\le \frac{R_{1,\max}}{1-\gamma}\|\nu-\mu\|_1.
\]
Hence,
\[
\|\nabla f(\nu)-\nabla f(\mu)\|_\infty
\le L_1\|\nu-\mu\|_1,
\qquad
L_1:=\frac{\omega_{\max}^2R_{1,\max}^2}{4\tau(1-\gamma)^2}.
\]

This $\ell_1$-smoothness implies the standard descent lemma:
\begin{equation}
\label{eq:descent-l1}
f(\nu)\le f(\mu)+\langle\nabla f(\mu),\nu-\mu\rangle+\frac{L_1}{2}\|\nu-\mu\|_1^2.
\end{equation}

\medskip
\textbf{(2) Controlling $\|\nu-\mu\|_1^2$ by $D_{\psi_\gamma}(\nu|\mu)$.}
The gradient of $\psi_\gamma$ is
$\nabla\psi_\gamma(\mu)(s,a)=\log\mu(s,a)-\log\rho_\mu(s)$, and one can verify that
\begin{equation}
\label{eq:Dpsi-conditionalKL}
D_{\psi_\gamma}(\nu|\mu)
=\sum_{s\in\mathcal S}\rho_\nu(s)\,
\mathrm{KL}\!\left(\pi_\nu(\cdot|s)\,\middle\|\,\pi_\mu(\cdot|s)\right).
\end{equation}

Using the occupancy-measure representation $\mu=\mu_{\pi_\mu}$ and $\nu=\mu_{\pi_\nu}$,
\[
\mu_{\pi}(s,a)=(1-\gamma)\sum_{t=0}^\infty \gamma^t\,d_t^\pi(s)\pi(a|s),
\qquad
\rho_{\pi}(s)=(1-\gamma)\sum_{t=0}^\infty \gamma^t\,d_t^\pi(s),
\]
where $d_t^\pi$ is the state distribution at time $t$.
Let $\Delta(s):=\|\pi_\nu(\cdot|s)-\pi_\mu(\cdot|s)\|_1$.
A standard argument using the non-expansiveness of Markov kernels in $\ell_1$ yields
\begin{equation}
\label{eq:mu-tv-policy-tv}
\|\nu-\mu\|_1
\le \frac{1}{1-\gamma}\sum_{s\in\mathcal S}\rho_\nu(s)\Delta(s).
\end{equation}
(Proof sketch: expand both occupancy measures as $(1-\gamma)\sum_t\gamma^t$ and bound the
distribution shift at each time step by the accumulated action-distribution shifts,
leading to the geometric factor $1/(1-\gamma)$.)

Applying Cauchy--Schwarz to \eqref{eq:mu-tv-policy-tv} gives
\[
\|\nu-\mu\|_1^2
\le \frac{1}{(1-\gamma)^2}
\left(\sum_{s}\rho_\nu(s)\Delta(s)\right)^2
\le \frac{1}{(1-\gamma)^2}\sum_{s}\rho_\nu(s)\Delta(s)^2.
\]
By Pinsker's inequality $\|\pi-\pi'\|_1^2\le 2\mathrm{KL}(\pi\|\pi')$ applied to each state $s$ and
\eqref{eq:Dpsi-conditionalKL}, we obtain
\[
\sum_s\rho_\nu(s)\Delta(s)^2
\le 2\sum_s\rho_\nu(s)\mathrm{KL}\!\left(\pi_\nu(\cdot|s)\,\middle\|\,\pi_\mu(\cdot|s)\right)
=2D_{\psi_\gamma}(\nu|\mu).
\]
Therefore,
\begin{equation}
\label{eq:l1-by-Dpsi}
\|\nu-\mu\|_1^2\le \frac{2}{(1-\gamma)^2}D_{\psi_\gamma}(\nu|\mu).
\end{equation}

\medskip
\textbf{(3) Combining the two inequalities.}
Substituting \eqref{eq:l1-by-Dpsi} into \eqref{eq:descent-l1} yields
\[
f(\nu)\le f(\mu)+\langle\nabla f(\mu),\nu-\mu\rangle
+\frac{L_1}{2}\cdot \frac{2}{(1-\gamma)^2}D_{\psi_\gamma}(\nu|\mu),
\]
that is,
\[
f(\nu)\le f(\mu)+\langle\nabla f(\mu),\nu-\mu\rangle
+L_{\mathrm{rel}}D_{\psi_\gamma}(\nu|\mu),
\qquad
L_{\mathrm{rel}}=\frac{L_1}{(1-\gamma)^2}
=\frac{\omega_{\max}^2R_{1,\max}^2}{4\tau(1-\gamma)^4}.
\]
This proves Proposition~\ref{prop:om-rel_smooth}. Since it is a concrete instantiation of
relative smoothness with reference function $\psi_\gamma$, it implies Lemma~\ref{lem:om-rel_smooth}.
\end{proof}

\section{Continuity Guarantee for Finite-Step Tabular CMDPI}
\label{app:continuity-omega-pi}

\begin{proposition}[Continuity of finite-step tabular CMDPI]
\label{prop:cmdpi_finite_step_continuity}
Consider a finite discounted multi-objective MDP with finite state and action
spaces $\mathcal S$ and $\mathcal A$, discount factor $\gamma\in(0,1)$, and
bounded vector-valued rewards. Let $\Omega \subset \mathbb R^m_{++}$ be a
preference set, and suppose that the scalarization $u(\mathbf J,\omega)$ is
continuously differentiable in $\mathbf J$ and that
\[
(\mathbf J,\omega)
\mapsto
\nabla_{\mathbf J} u(\mathbf J,\omega)
\]
is continuous on its domain. Assume that the initialization
$\omega\mapsto \pi_0(\omega)$ is continuous and has full support, i.e.,
$\pi_0(a\mid s;\omega)>0$ for all $(s,a,\omega)$.

For fixed $\alpha>0$, suppose that the exact tabular CMDPI recursion defines
$\pi_{k+1}$ from $\pi_k$ as follows. First, set
\[
\mathbf g_k(\omega)
=
\nabla_{\mathbf J}
u\!\left(\mathbf J(\pi_k(\omega)),\omega\right),
\]
and define the scalarized reward
\[
r_k^\omega(s,a)
=
\mathbf g_k(\omega)^\top \mathbf r(s,a).
\]
Then let $Q_{k+1}^{\omega}$ be the unique fixed point of the soft Bellman
equation
\[
Q(s,a)
=
r_k^\omega(s,a)
+
\gamma
\mathbb E_{s'\sim P(\cdot\mid s,a)}
\left[
\chi\!\left(s';Q,\pi_k(\omega)\right)
\right],
\]
where
\[
\chi(s;Q,\pi)
=
\alpha
\log
\sum_{b\in\mathcal A}
\pi(b\mid s)
\exp\!\left(\frac{Q(s,b)}{\alpha}\right).
\]
Finally, update the policy by
\[
\pi_{k+1}(a\mid s;\omega)
=
\frac{
\pi_k(a\mid s;\omega)
\exp\!\left(Q_{k+1}^{\omega}(s,a)/\alpha\right)
}{
\sum_{b\in\mathcal A}
\pi_k(b\mid s;\omega)
\exp\!\left(Q_{k+1}^{\omega}(s,b)/\alpha\right)
}.
\]
Then, for every finite $k\ge 0$, the map
\[
\omega \mapsto \pi_k(\omega)
\]
is continuous.
\end{proposition}

\begin{proof}
We prove the claim by induction on $k$. 
The induction step consists of the following three parts:
\begin{itemize}
    \item[(i)] continuity of the scalarized reward
    $\omega\mapsto r_k^\omega$;
    \item[(ii)] continuity of the Bellman fixed point
    $\omega\mapsto Q_{k+1}^{\omega}$;
    \item[(iii)] continuity of the updated policy
    $\omega\mapsto\pi_{k+1}(\omega)$.
\end{itemize}
The only nontrivial step is (ii), because $Q_{k+1}^{\omega}$ is defined
implicitly as the fixed point of a $\omega$-dependent Bellman operator. We
handle this step using a standard perturbation argument for parameterized
contractions.

The base case $k=0$ holds by the assumption that
$\omega\mapsto\pi_0(\omega)$ is continuous.

\paragraph{(i) Continuity of the scalarized reward.}
Assume that $\omega\mapsto\pi_k(\omega)$ is continuous. Since the state and
action spaces are finite and $\gamma<1$, the objective vector
$\mathbf J(\pi)$ is continuous as a function of the tabular policy $\pi$.
Hence
\[
\omega
\mapsto
\mathbf g_k(\omega)
=
\nabla_{\mathbf J}
u\!\left(\mathbf J(\pi_k(\omega)),\omega\right)
\]
is continuous. Since the vector reward is bounded and finite-dimensional,
$\omega\mapsto r_k^\omega$ is also continuous.

\paragraph{(ii) Continuity of the Bellman fixed point.}
For fixed $\pi$, the map
\[
Q
\mapsto
\chi(s;Q,\pi)
\]
is $1$-Lipschitz with respect to the sup norm. Therefore, the soft Bellman
operator
\[
(T_{k,\omega}Q)(s,a)
:=
r_k^\omega(s,a)
+
\gamma
\mathbb E_{s'\sim P(\cdot\mid s,a)}
\left[
\chi\!\left(s';Q,\pi_k(\omega)\right)
\right]
\]
is a $\gamma$-contraction in $Q$. Thus, for each $\omega$, it has a unique
fixed point $Q_{k+1}^{\omega}$.

It remains to show that $\omega\mapsto Q_{k+1}^{\omega}$ is continuous.
Let $\omega'\to\omega$. Since $Q_{k+1}^{\omega'}$ and
$Q_{k+1}^{\omega}$ are fixed points of $T_{k,\omega'}$ and
$T_{k,\omega}$, respectively, and all operators have the same contraction
modulus $\gamma$, we have
\[
\begin{aligned}
\left\|
Q_{k+1}^{\omega'}
-
Q_{k+1}^{\omega}
\right\|_\infty
&=
\left\|
T_{k,\omega'} Q_{k+1}^{\omega'}
-
T_{k,\omega} Q_{k+1}^{\omega}
\right\|_\infty
\\
&\le
\left\|
T_{k,\omega'} Q_{k+1}^{\omega'}
-
T_{k,\omega'} Q_{k+1}^{\omega}
\right\|_\infty
+
\left\|
T_{k,\omega'} Q_{k+1}^{\omega}
-
T_{k,\omega} Q_{k+1}^{\omega}
\right\|_\infty
\\
&\le
\gamma
\left\|
Q_{k+1}^{\omega'}
-
Q_{k+1}^{\omega}
\right\|_\infty
+
\left\|
T_{k,\omega'} Q_{k+1}^{\omega}
-
T_{k,\omega} Q_{k+1}^{\omega}
\right\|_\infty .
\end{aligned}
\]
Rearranging gives
\[
\left\|
Q_{k+1}^{\omega'}
-
Q_{k+1}^{\omega}
\right\|_\infty
\le
\frac{1}{1-\gamma}
\left\|
T_{k,\omega'} Q_{k+1}^{\omega}
-
T_{k,\omega} Q_{k+1}^{\omega}
\right\|_\infty .
\]
Here the right-hand side is evaluated at the fixed vector
$Q_{k+1}^{\omega}$. Therefore, it is enough to show that
$\omega\mapsto T_{k,\omega}Q$ is continuous for each fixed $Q$.

This follows from the continuity of $\omega\mapsto r_k^\omega$, the induction
hypothesis that $\omega\mapsto\pi_k(\omega)$ is continuous, and the joint
continuity of $\chi(s;Q,\pi)$ in $(Q,\pi)$. Hence
\[
\left\|
T_{k,\omega'} Q_{k+1}^{\omega}
-
T_{k,\omega} Q_{k+1}^{\omega}
\right\|_\infty
\to 0
\qquad
(\omega'\to\omega),
\]
and consequently
\[
Q_{k+1}^{\omega'}
\to
Q_{k+1}^{\omega}.
\]

\paragraph{(iii) Continuity of the multiplicative policy update.}
Finally, the policy update
\[
\pi_{k+1}(a\mid s;\omega)
=
\frac{
\pi_k(a\mid s;\omega)
\exp\!\left(Q_{k+1}^{\omega}(s,a)/\alpha\right)
}{
\sum_{b\in\mathcal A}
\pi_k(b\mid s;\omega)
\exp\!\left(Q_{k+1}^{\omega}(s,b)/\alpha\right)
}
\]
is a ratio of continuous functions. The denominator is strictly positive
because $\pi_k(\cdot\mid s;\omega)$ has full support and the exponential term
is strictly positive. Therefore,
$\omega\mapsto\pi_{k+1}(\omega)$ is continuous. Moreover, the same update
preserves full support, so the induction continues. This proves the claim for
all finite $k$.
\end{proof}

\section{Implementation Details of Preference-Conditioned Single Actor-Critic}
\label{app:impl}

\begin{algorithm}[tb]
\caption{Preference-conditioned Single Actor-Critic}
\label{alg:pcsac}
\begin{algorithmic}
\REQUIRE Preference-conditioned actor: $\pi_\theta$ and its target actor $\pi_{\bar{\theta}}$, Preference-conditioned vector critic ($m$-dimension outputs): $\mathbf{Q}_\phi(s, a, \omega)$ and its target network, Replay buffer $\mathcal{D}$, and parameters $\alpha$, $\tau$
\FOR{$k=0,1,\cdots$}
    \STATE Store samples with $\pi_\theta$ into the $\mathcal{D}$
    \STATE $(s, a, \mathbf{r}, s') \sim \mathcal{D}$ and sample $\omega$ uniformly
    \STATE $a' \sim \pi_\theta(\cdot|s', \omega)$
    \STATE $\mathbf{y} = \mathbf{r} + \gamma \mathbf{Q}_{\bar{\phi}}(s', a', \omega) - \alpha \log \pi_\theta(a'|s')$
    \STATE $L_\phi = \frac{1}{2}\|\mathbf{y} - \mathbf{Q}_{\phi}(s, a, \omega)\|_2^2$
    \STATE $\tilde{a} \sim \pi_\theta(\cdot|s, \omega)$
    \STATE $\mathbf{g} = \nabla u(\mathbf{Q}_{\phi}(s, \tilde{a}, \omega), \omega)$
    \STATE $L_\theta = - \mathbf{g}^\top \mathbf{Q}_{\phi}(s, \tilde{a}, \omega)$ with 
    \STATE update $\pi_\theta$ based on the gradient of $L_\theta$ and $\mathbf{Q}_{\phi}$ based on $L_\phi$
\ENDFOR
\end{algorithmic}
\end{algorithm}

\paragraph{Base implementations.}
We implement both the model-free (MF) and model-based (MB) variants on top of open-source codebases.
As MF baseline, we start from Soft Actor-Critic (SAC) \citep{haarnoja2018sac}.
As MB baseline, we start from the DHMBPO implementation \citep{kubo2025double}, which combines MBPO \citep{janner2019trust} with SVG-style policy optimization \citep{pmlr-v144-amos21a} and uses deep ensembles for learned dynamics.
Importantly, the DHMBPO codebase already contains a SAC implementation; hence, we keep the network architectures (actor/critic), optimizer choices, and most training hyperparameters identical across MF/MB, and only change the components required for multi-objective preference-conditioning and scalarization.

\paragraph{Preference-conditioned actor and vector critic.}
We modify the actor and critic networks so that they are conditioned on a preference vector $\omega$.
The actor is $\pi_\theta(a\mid s,\omega)$, and the critic outputs an $m$-dimensional vector
$\mathbf{Q}_\phi(s,a,\omega)\in\mathbb{R}^m$.
In all our MuJoCo experiments, $m=2$ (see Appendix~\ref{app:exp} for the environment list).

\paragraph{Preference sampling during training.}
During data collection in the real environment, we sample a preference vector $\omega$ at the beginning of each episode and keep it fixed within the episode.
During each policy optimization step (including model rollouts in DHMBPO, i.e., Distribution rollout and Training rollout), we additionally sample a fresh $\omega$ from the same distribution whenever the actor/critic is evaluated.
Concretely, $\omega$ is sampled using PyTorch \texttt{RelaxedOneHotCategorical} with temperature $1$ and uniform class probabilities $\mathbf{1}/m$.
This sampling can yield near-zero (and in principle zero) coordinates, but we did not observe numerical errors in our experiments.

\paragraph{Vector SAC targets.}
Given a mini-batch $(s,a,\mathbf{r},s')$ from the replay buffer and a sampled preference $\omega$, we compute the target action $a'\sim\pi_\theta(\cdot\mid s',\omega)$ and define the vector target
\begin{equation}
\mathbf{y}
=
\mathbf{r}
+
\gamma\,\mathbf{Q}_{\bar{\phi}}(s',a',\omega)
-
\alpha\,\log \pi_\theta(a'\mid s',\omega)\,\mathbf{1},
\end{equation}
where $\mathbf{1}\in\mathbb{R}^m$ broadcasts the entropy term to each reward dimension.
We use the squared loss $L_\phi=\tfrac12\|\mathbf{y}-\mathbf{Q}_\phi(s,a,\omega)\|_2^2$.

\paragraph{Scalarized actor update via utility gradients.}
To update the actor, we convert the vector critic output into a scalar objective using the gradient of the utility.
Let $\tilde{a}\sim \pi_\theta(\cdot\mid s,\omega)$ and $\mathbf{z}=\mathbf{Q}_\phi(s,\tilde{a},\omega)$.
We compute
\begin{equation}
\mathbf{g} = \nabla_{\mathbf{z}} u(\mathbf{z},\omega),
\end{equation}
and use the following actor loss:
\begin{equation}
L_\theta = - \mathbf{g}^\top \mathbf{Q}_\phi(s,\tilde{a},\omega).
\end{equation}
For linear scalarization, $\mathbf{g}=\omega$ and this recovers the standard linear weighted objective.
For the proposed STCH scalarization, $\mathbf{g}$ depends on $\mathbf{z}$ and $\omega$ through $u$.
In our implementation, we treat $\mathbf{g}$ as a constant with respect to $\theta$ (i.e., we stop gradients through $\mathbf{g}$) to avoid second-order derivatives.
The complete training loop is summarized in Algorithm~\ref{alg:pcsac}.

\paragraph{Reward normalization and a utopia point estimation.}
We normalize rewards to stabilize training across objectives and environments.
At each policy optimization step, we compute the sample mean and standard deviation of \emph{all} real-environment reward samples collected up to that time, and normalize each reward component in the sampled mini-batch accordingly.
Using the normalized mini-batch rewards, we estimate an ``ideal point'' $\mathbf{I}\in\mathbb{R}^m$ as follows:
we compute the 99th percentile (component-wise) of the normalized rewards within the mini-batch,
multiply it by $1/(1-\gamma)$, and then update $\mathbf{I}$ by an exponential moving average (EMA)
\begin{equation}
\mathbf{I} \leftarrow (1-\beta)\mathbf{I} + \beta\,\widehat{\mathbf{I}},
\qquad \beta=3\times 10^{-4},
\end{equation}
with initialization $\mathbf{I}=\mathbf{1}$.
Similarly, we compute the 1st percentile and maintain its EMA to obtain a lower bound vector.
When computing the critic targets, we clip the target critic output using these bounds (this clipping is inherited from the DHMBPO implementation).

\paragraph{Temperature parameters.}
Our method introduces a single additional hyperparameter, the PCSAC temperature $\tau$, and we test $\tau\in\{0.1,0.01,0.001\}$.
We fix SAC's entropy temperature $\alpha$ to $0.01$ for all environments, motivated by the fact that we normalize rewards (and thus avoid environment-specific tuning of $\alpha$).

\paragraph{Hyperparameters.}
All remaining hyperparameters are shared with DHMBPO \citep{kubo2025double} unless otherwise stated.
Table~\ref{tab:hyper} lists the values used in our experiments.

\section{Detail of Numerical Experiments}
\subsection{Planning experiment details for the toy tabular MOMDP}
\label{app:exp-toy}

\begin{figure}[tb]
    \centering
    \includegraphics[width=0.5\linewidth]{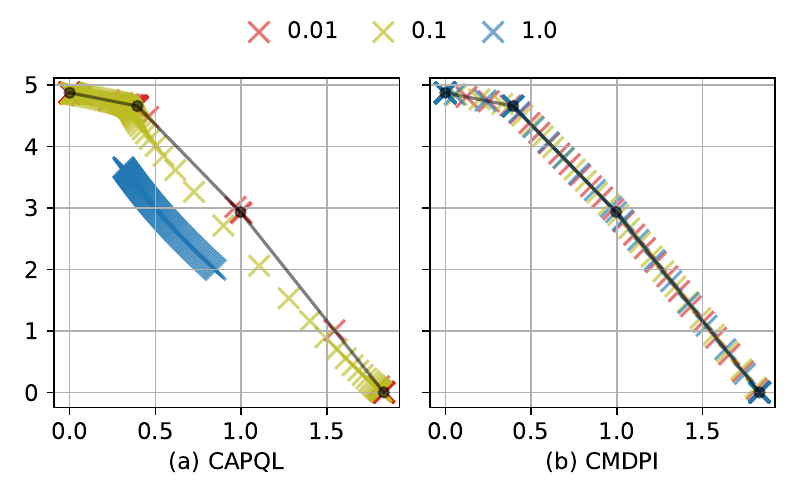}
    \caption{
    Scatter plots of the converged objective vectors in a two-objective MOMDP, obtained by (a) the regularized method CAPQL~\citep{lu2023multiobjective} and (b) the proposed method  CMDPI.
    The black solid curve indicates the Pareto front, and the black dots indicate its vertices, which are recovered by the linear-utility baseline.
    The legend is shared across the two panels because we use the same set of hyperparameter values for CAPQL ($\alpha$) and our method ($\tau$), i.e., we pair $\alpha = 1/\tau$ and sweep the same three values in both panels.
    CAPQL exhibits a trade-off between achieving a dense front and accurately recovering a subset of Pareto-optimal points.
    In contrast, our method recovers Pareto-optimal solutions accurately and more uniformly (visually denser) across all tested hyperparameter values.
    }
    \label{fig:tabluar}
\end{figure}

We report implementation details of the planning-based comparison in the toy two-objective MOMDP used in Fig.~\ref{fig:tabluar}.
All methods are executed in the \emph{planning} setting: the transition kernel $P$ and reward functions $\mathbf r$ are fully known,
the horizon is infinite with discount $\gamma=0.8$, and all quantities (e.g., $\mathbf J(\pi)$) are computed exactly by matrix operations.
No sampling-based learning is involved; the only stochasticity comes from the random initialization of policies for iterative methods.

We summarize the CMDPI algorithm in Algorithm~\ref{alg:cmdpi}.
\begin{algorithm}[tb]
\caption{Concave Mirror Descent Policy Iteration}
\label{alg:cmdpi}
\begin{algorithmic}
\REQUIRE preference $\omega\in\mathbb R^m_{++}$, temperature $\alpha>0$, initial policy $\pi_0$ with full support.
\FOR{$k=0,1,\dots,K-1$}
    \STATE Compute/estimate $\mathbf J(\pi_k)$. \hfill (objective evaluation)
    \STATE \textbf{Dynamic scalarization weight:} $\mathbf g_k=\nabla_{\mathbf J}u\big(\mathbf J(\pi_k),\omega\big)\in\mathbb R^m_{++}$.
    \STATE \textbf{Scalar reward:} $r_k(s,a)=\langle \mathbf g_k,\mathbf r(s,a)\rangle$.
    \STATE \textbf{Soft policy evaluation:} solve the fixed point~\eqref{eq:soft-bellman} to obtain $Q_{k+1}$.
    \STATE \textbf{Policy improvement:} update $\pi_{k+1}$ by~\eqref{eq:pi-update}.
\ENDFOR
\end{algorithmic}
\end{algorithm}

\paragraph{Environment.}
We consider a discounted MOMDP with $\mathcal S=\{0,1,2,3\}$ and $\mathcal A=\{0,1\}$.
The initial distribution is uniform: $p_0(s)=1/|\mathcal S|$ for all $s$.
The transition kernel is specified by two row-stochastic matrices $P^a\in\mathbb R^{|\mathcal S|\times|\mathcal S|}$ where
$P^a_{s,s'}=P(s'\mid s,a)$.
For action $a=0$,
\[
P^{0}=
\begin{bmatrix}
1 & 0 & 0 & 0\\
1 & 0 & 0 & 0\\
0 & 1 & 0 & 0\\
0 & 0 & 1 & 0
\end{bmatrix},
\]
and for action $a=1$,
\[
P^{1}=
\begin{bmatrix}
0.1 & 0.9 & 0 & 0\\
0.1 & 0 & 0.9 & 0\\
0 & 0.1 & 0 & 0.9\\
0 & 0 & 0.1 & 0.9
\end{bmatrix}.
\]
Intuitively, action $0$ moves ``left'' deterministically (with reflection at the left boundary),
and action $1$ moves ``right'' with probability $0.9$ (and ``left'' with probability $0.1$), with reflection at the right boundary.

The two-dimensional reward vector is $\mathbf r(s,a)=(r_1(s,a),r_2(s,a))\in\mathbb R^2$ and is given by
\[
\mathbf r(s,0)=
\begin{cases}
(0,0) & s=0\\
(0,2) & s=1\\
(0,1.6) & s=2\\
(1,0) & s=3
\end{cases}
\qquad
\mathbf r(s,1)=
\begin{cases}
(0,0) & s=0\\
(0,0) & s=1\\
(0,0) & s=2\\
(0,1) & s=3.
\end{cases}
\]
This parameterization is chosen so that the attainable objective set is visibly nontrivial and the Pareto boundary is not axis-aligned.

\paragraph{Preference vectors and STCH reference point.}
We use $100$ preference vectors on the simplex, equally spaced and including endpoints:
\[
\omega^{(i)}=\Bigl(\frac{i}{99},\,1-\frac{i}{99}\Bigr),\quad i=0,1,\dots,99.
\]
Although the theory often assumes $\omega\in\mathbb R^2_{++}$, in this environment the Pareto boundary has no segments parallel to coordinate axes,
so including endpoints $\omega^{(0)}=(0,1)$ and $\omega^{(99)}=(1,0)$ does not introduce degeneracy in practice.

For the STCH utility, we set the reference point componentwise as
\[
I_\ell = \max_{s\in\mathcal S,\ a\in\mathcal A}\left|\frac{r_\ell(s,a)}{1-\gamma}\right|,\quad \ell\in\{1,2\}.
\]
With $\gamma=0.8$ and the above rewards, this yields $I_1=\frac{1}{0.2}=5$ and $I_2=\frac{2}{0.2}=10$.

\paragraph{Compared planning methods.}
For each preference vector $\omega^{(i)}$, we independently run the following methods and obtain a stationary policy $\pi$.

\emph{(i) Linear scalarization + value iteration (reference points/curve).}
We define the scalar reward $r_\omega(s,a)=\langle \omega, \mathbf r(s,a)\rangle$ and solve
$\max_{\pi\in\Pi}\mathbb E_{s\sim p_0}[V^\pi_\omega(s)]$ by standard discounted value iteration in the tabular model.
Let $\pi^{\mathrm{lin}}_{\omega}$ be an optimal policy and $\mathbf J(\pi^{\mathrm{lin}}_{\omega})$ its (vector) objective.
The black dots in Fig.~\ref{fig:tabluar} indicate the distinct objective vectors obtained by this linear-scalarization planning
(which correspond to supported extreme points in this toy problem).
The black solid curve is drawn by sorting these points along the boundary and connecting them by piecewise-linear interpolation.

\emph{(ii) CAPQL.}
We implement the CAPQL planning update of \citet{lu2023multiobjective} in the tabular setting with known $P$ and $\mathbf r$,
and iterate the associated operator until convergence to obtain a stationary policy.
We sweep three regularization strengths $\alpha$ (as shown in Fig.~\ref{fig:tabluar}).

\emph{(iii) Proposed CMDPI under STCH utility.}
We run the proposed mirror-descent policy iteration algorithm for maximizing the STCH utility
$u(\mathbf J(\pi),\omega)$ with smoothing parameter $\tau>0$.
Following the implementation used in the main experiments, we set the CMDPI step parameter as $\alpha = 1/\tau$.
We sweep three values of $\tau$ (as shown in Fig.~\ref{fig:tabluar}).

For CAPQL and the proposed method, policies are initialized randomly.
In our implementation, for each state $s$ we sample $\pi_0(\cdot\mid s)$ independently from the uniform distribution over the simplex
(equivalently, $\pi_0(\cdot\mid s)\sim\mathrm{Dirichlet}(1,1)$), and then run deterministic planning iterations.

\paragraph{Exact evaluation of objective vectors.}
For every returned stationary policy $\pi$, we compute its objective vector $\mathbf J(\pi)$ exactly (no rollouts).
Let $P_\pi\in\mathbb R^{|\mathcal S|\times|\mathcal S|}$ denote the induced transition matrix
$P_\pi(s,s')=\sum_{a}\pi(a\mid s)P(s'\mid s,a)$.
We compute the discounted state occupancy $\rho_\pi\in\mathbb R^{|\mathcal S|}$ as the unique solution of
\[
\rho_\pi^\top = (1-\gamma)p_0^\top + \gamma \rho_\pi^\top P_\pi,
\]
and then form the discounted occupancy measure $\mu_\pi(s,a)=\rho_\pi(s)\pi(a\mid s)$.
Finally, we compute
\[
\mathbf J(\pi)=\frac{1}{1-\gamma}\sum_{s\in\mathcal S}\sum_{a\in\mathcal A}\mu_\pi(s,a)\,\mathbf r(s,a),
\]
which matches the definition used throughout the paper.

\paragraph{Visualization.}
For each method and each hyperparameter setting (three values for CAPQL and three values for the proposed method),
we plot the $100$ exact objective vectors $\{\mathbf J(\pi_{\omega^{(i)}})\}_{i=0}^{99}$ as a scatter plot,
where different colors correspond to different hyperparameters.
The black dots and the black solid line are produced by linear scalarization + value iteration as described above.

\subsection{Experimental Protocol and Baselines on MO-Gymnasium}
\label{app:exp}

\subsubsection{Experimental Protocol for Discrete-Action MO-Gymnasium Tasks}
\label{app:exp-discrete}

\paragraph{Environments.}
We evaluate on eight MO-Gymnasium~\cite{felten2023toolkit} tasks with discrete action spaces:
\texttt{deep-sea-treasure-concave-v0},
\texttt{deep-sea-treasure-v0},
\texttt{fishwood-v0},
\texttt{four-room-v0},
\texttt{fruit-tree-v0},
\texttt{minecart-v0},
\texttt{mo-lunar-lander-v3}, and
\texttt{mo-reacher-v5}.
The reward dimension ranges from two to six.
We use the default episode-length settings of MO-Gymnasium and train each method for $2$M environment interaction steps.

\paragraph{Baselines and codebases.}
We compare LS (linear scalarization), CAPQL~\cite{lu2023multiobjective}, PreCo~\cite{yangpreference}, C-MORL~\cite{liu2025cmorl}, PCSAC, and CMDPI.
LS, CAPQL, PreCo, PCSAC, and CMDPI are run within the same PreCo's official codebase to reduce implementation-level confounders.
PCSAC and CMDPI are our STCH-based preference-conditioned actor--critic variants.
For the PreCo-based implementations, we apply a common per-objective running reward normalization in critic learning, using the empirical mean and standard deviation of collected reward vectors.
This normalization is used only during training; all reported return vectors and metrics are computed from unnormalized episodic returns.

C-MORL is included as a strong multi-policy Pareto-front discovery baseline, although it is not a single preference-conditioned policy.
We use the official public implementation and modify the interaction accounting and logging needed to align the training budget and post-hoc metric computation with the other methods.

\paragraph{Training seeds.}
We run eight random seeds for each method and environment.

\paragraph{Evaluation and metric computation.}
At each evaluation checkpoint, we collect unnormalized episodic return vectors from each method and save them for post-hoc evaluation.
For preference-conditioned methods, we perform preference-sweep evaluation by conditioning the policy on a fixed rollout preference set and recording the resulting return vectors.
For C-MORL, which outputs a set of policies rather than a single preference-conditioned policy, we evaluate the policies in the discovered set to obtain return vectors.

We do not use the metric values logged by the original C-MORL implementation.
Instead, after all runs are completed, we compute HV, EUM, and SP for all methods from the saved unnormalized return vectors using the same post-processing script.
HV is computed with a common environment-wise reference point.
EUM is computed using a common EUM evaluation preference set shared by all methods.

\paragraph{Evaluation preferences.}
We distinguish between two preference sets.
The rollout preference set is used as input to preference-conditioned policies during evaluation rollouts and therefore applies only to LS, CAPQL, PreCo, PCSAC, and CMDPI.
The EUM evaluation preference set is used only for metric aggregation and is shared by all methods, including C-MORL.
For C-MORL, the EUM preferences are used only at the aggregation stage because C-MORL returns a policy set rather than a preference-conditioned policy.
All preference sets are generated by the deterministic simplex-grid routine used in PreCo-based codebase and are fixed across methods, checkpoints, and seeds.

\paragraph{Reference points.}
We distinguish between reference points used for reported HV computation and reference values used internally by C-MORL.
For reported HV, we use a common environment-wise reference point shared by all methods.
Following the same protocol as in the continuous-action experiments, we determine this reference point a posteriori as the component-wise minimum over the unnormalized evaluation return vectors collected across all methods, checkpoints, and random seeds.

C-MORL additionally uses internal reference values in its online multi-policy selection procedure.
These internal values are used only during C-MORL training and are never used to compute the reported HV, EUM, or SP values.
When the official C-MORL implementation provides environment-specific internal values, we keep them; otherwise, we use a conservative lower-bound vector with all entries set to $-10^4$.

\paragraph{Counting environment interactions.}
We count only environment interactions used for training and exclude evaluation rollouts for all methods.
For C-MORL, we count interactions across parallel sampling workers when necessary, so that the logged sampling steps reflect the actual number of environment transitions used for training.
This adjustment affects only sample-efficiency accounting.

\paragraph{Hyperparameters.}
PCSAC uses the off-policy replay-buffer optimization scheme with entropy coefficient $\alpha_{\mathrm{ent}}=0.3$ tuned out of $\{0.03, 0.1, 0.3\}$.
CMDPI uses the on-policy optimization scheme motivated by the KL-regularized formulation, with tuned smoothness/KL parameter 0.001 out of  $\{0.1, 0.01, 0.001\}$.
For numerical stability of the KL-regularized actor update, the CMDPI actor optimizer uses Adam with weight decay $10^{-4}$.
For C-MORL, we tune the fraction of the total $2$M-step budget allocated to the initialization stage 0.4 out of $\{0.2, 0.4, 0.5\}$.
All other shared optimizer and network hyperparameters in the PreCo-based implementations follow the default settings of the codebase.

\subsubsection{Experimental Protocol for Continuous-Action MO-Gymnasium Tasks}
\label{app:exp-continuous}

\paragraph{Environments.}
We evaluate our method on eight continuous-control multi-objective MuJoCo
environments from MO-Gymnasium (all v5).
All tasks are configured with two or three reward dimensions.
Specifically, we use
\texttt{mo-halfcheetah-v5},
\texttt{mo-hopper-v5},
\texttt{mo-swimmer-v5},
\texttt{mo-walker2d-v5},
and truncated-observation variants of Ant and Humanoid, denoted
\texttt{TruncAnt} and \texttt{TruncHumanoid}, following the standard
``truncated state'' setting used in previous works~\citep{chua2018deep,janner2019trust,kubo2025double}.
For the Pareto-front visualization experiments, we use two-objective versions
of \texttt{mo-hopper-v5} and \texttt{TruncAnt} by removing one reward dimension
from their original reward vectors.

\paragraph{Baselines and codebases.}
Besides our MF/MB baselines built on SAC/DHMBPO, we run additional state-of-the-art MORL baselines using their publicly available implementations.
We note that CAPQL and our MF-Linear baseline are conceptually similar (both are preference-conditioned actor-critic with linear scalarization), but differ in codebase and in our additional reward normalization described in Appendix~\ref{app:impl}.

\paragraph{Training seeds.}
For our baselines and proposed method, we run $8$ random seeds for each environment.

\paragraph{Reference point}
We set for each environments reference point by taking the minimum expected reward for each dimension over all runs; all over the methods and preference-sweep evaluation steps and seeds.

\paragraph{Evaluation schedule (preference-sweep evaluation).}
During training, we periodically perform \emph{preference-sweep evaluation}.
We evaluate every $25{,}000$ real-environment interaction steps by default.
Some external implementations cannot be paused at exactly multiples of $25{,}000$ steps; in those cases, we evaluate at the first available checkpoint after each multiple.
For plotting sample-efficiency curves, we post-process the logged evaluations by binning them into $25{,}000$-step bins:

\paragraph{Evaluation preferences.}
For each environment with $m=2$ objectives, we pre-generate 20 approximately equally spaced preference vectors on the simplex using
\texttt{morl\_baselines.common.weights.equally\_spaced\_weights} from MORL-Baselines~\citep{felten2023toolkit}.
At a given training checkpoint and seed, for each preference vector $\omega$,
we condition the actor on $\omega$ and roll out the policy.
We record the undiscounted episodic return vectors and average them over episodes to obtain an estimate of the expected return vector for that $(\text{checkpoint},\text{seed})$.

\paragraph{Counting environment interactions (sample-efficiency).}
To compare sample-efficiency fairly across methods, we count \emph{all} real-environment interaction steps used for training, excluding only evaluation rollouts.
For model-based methods, imaginary rollouts inside learned models are \emph{not} counted.
We also do not count gradient update steps or replay-buffer sampling frequency (i.e., the update-to-data ratio does not change the step count).

\paragraph{Accounting for additional interaction phases in baselines.}
Some baselines use extra real-environment interaction phases beyond standard off-policy training.
For GPI-PD \citep{alegre2023sample}, we include the real-environment interactions used to identify the optimal preference vector as part of the sample count.
For PD-MORL \citep{basaklar2022pd}, the method relies on a \emph{key preference set} for preference alignment:
it defines key preferences as the unit vectors (one per objective) plus the uniform preference (e.g., for $m=2$: $\{[1,0],[0.5,0.5],[0,1]\}$),
trains using fixed key preferences to obtain corresponding solutions, and uses them to normalize the solution space and fit an interpolator that projects preferences for alignment.
We follow the official PD-MORL procedure and count the real-environment interactions used in this key-preference phase (200K steps in our setup) as part of the sample count.

\paragraph{Hyperparameters.}
Table~\ref{tab:hyper} lists the hyperparameters shared across our MF/MB baselines and the proposed method.
All unspecified parameters follow DHMBPO \citep{kubo2025double}.

\begin{table}[tb]
    \centering
    \caption{Hyperparameters for our SAC/DHMBPO-based implementations (MF/MB Linear and MF/MB PCSAC).
    From top to bottom: (i) the PCSAC-specific parameter shared by our MF/MB variants, (ii) hyperparameters shared across our MF and MB implementations, and (iii) MB-only hyperparameters shared across our MB implementations.
    }
    \label{tab:hyper}

    \begin{tabular}{cc}
    \toprule
    Hyper-parameter & Value \\
    \midrule
    PCSAC temperature $\tau$ & $\{0.1, 0.01, 0.001\}$ \\
    \midrule
    Discount factor & 0.995 \\
    Seed steps & 5000 \\
    Batch size & 256 \\
    Update-to-data ratio & 1 \\
    Replay buffer size & 1M \\
    Learning rate for the actor, critics and $\alpha$ & $3\cdot10^{-4}$ \\
    Fixed value of $\alpha$ & 0.01 \\
    Momentum coefficient $c$ for target critic & 0.995 \\
    Ensemble size of critic & 5 \\
    \midrule
    Length of distribution rollout  & 20 \\
    Length of training rollout  & 5 \\
    Ensemble size of model  & 8 \\
    Optimizer for training model  & AdamW \citep{loshchilov2018decoupled} \\
    Learning rate for model  & $1\cdot10^{-3}$ \\
    \bottomrule
    \end{tabular}

\end{table}

\subsection{Detailed Results}
\subsubsection{Discrete action tasks}
\label{app:exp-detail-discrete}

\begin{figure}[tb]
    \centering
    \begin{subfigure}[t]{0.32\linewidth}
        \centering
        \includegraphics[width=\linewidth]{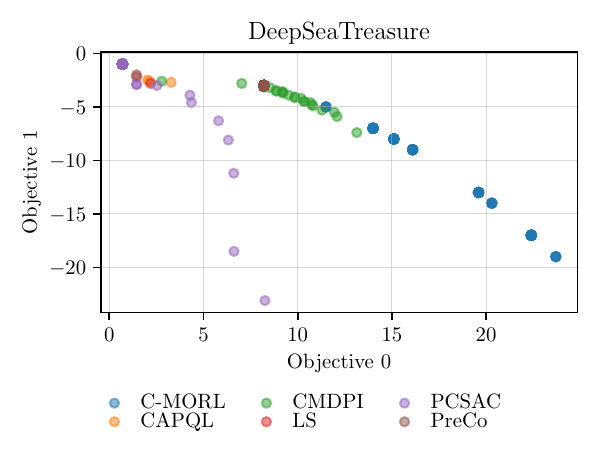}
        \caption{\texttt{DeepSeaTreasure}}
        \label{fig:disc-scatter-task1}
    \end{subfigure}
    \hfill
    \begin{subfigure}[t]{0.32\linewidth}
        \centering
        \includegraphics[width=\linewidth]{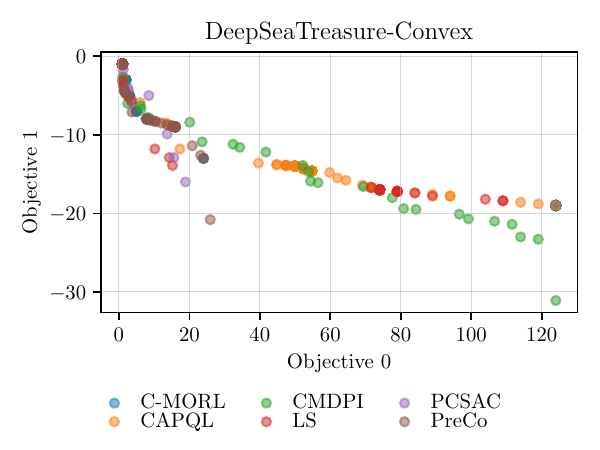}
        \caption{\texttt{DeepSeaTreasure-Concave}}
        \label{fig:disc-scatter-task2}
    \end{subfigure}
    \hfill
    \begin{subfigure}[t]{0.32\linewidth}
        \centering
        \includegraphics[width=\linewidth]{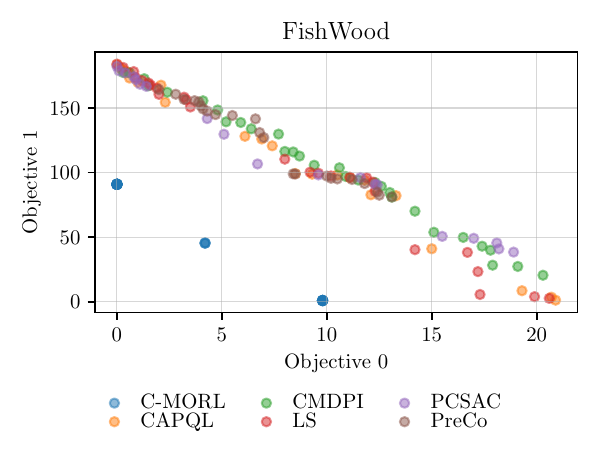}
        \caption{\texttt{FishWood}}
        \label{fig:disc-scatter-task3}
    \end{subfigure}

    \vspace{0.5em}

    \begin{subfigure}[t]{0.32\linewidth}
        \centering
        \includegraphics[width=\linewidth]{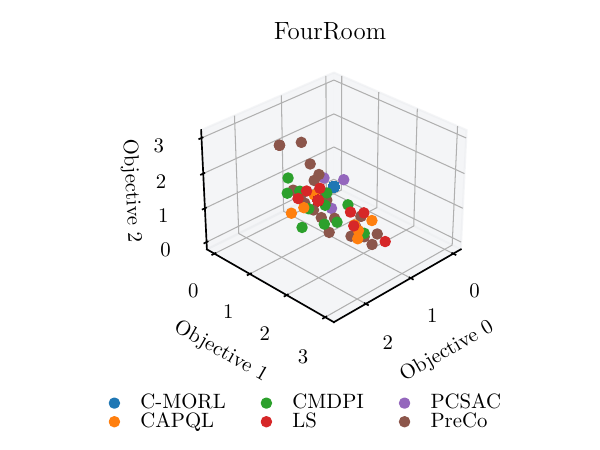}
        \caption{\texttt{FourRoom}}
        \label{fig:disc-scatter-task4}
    \end{subfigure}
    \begin{subfigure}[t]{0.32\linewidth}
        \centering
        \includegraphics[width=\linewidth]{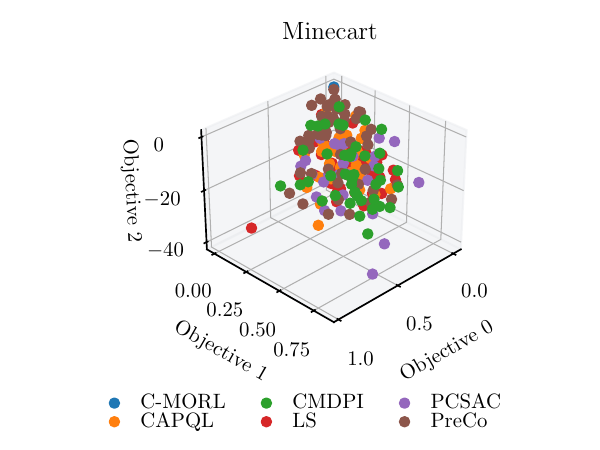}
        \caption{\texttt{Minecart}}
        \label{fig:disc-scatter-task5}
    \end{subfigure}
    \hfill

    \caption{
    Final-step return scatter plots for five two- or three-objective tasks.
    }
    \label{fig:disc-scatter-plots}
\end{figure}

Figures~\ref{fig:disc-scatter-plots} present the final-step return scatter plots 
for five two- or three-objective tasks.
In addition, Table~\ref{tbl:disc_each} reports task-wise metrics for the
eight discrete-action MO-Gymnasium benchmarks.

\begin{table}
\caption{Comparison of final-step metrics, EUM, HV, and Sparsity on discrete action tasks from MO-gymnasium .}
\label{tbl:disc_each}
\resizebox{\textwidth}{!}{%
\begin{tabular}{lc|cccccc}
\toprule
Environments & Metrics & LS & CAPQL & PreCo & C-MORL & PCSAC & Deep CMDPI \\
\midrule
\multirow[c]{3}{*}{DST} & EUM $\uparrow$ ($10^{0}$) & $2.60 \pm 0.00$ & $2.60 \pm 0.00$ & $2.60 \pm 0.00$ & $\mathbf{6.68 \pm 0.10}$ & $-41.83 \pm 6.10$ & $3.18 \pm 0.10$ \\
 & HV $\uparrow$ ($10^{3}$) & $0.80 \pm 0.00$ & $0.80 \pm 0.00$ & $0.80 \pm 0.00$ & $\mathbf{1.99 \pm 0.03}$ & $0.01 \pm 0.01$ & $1.01 \pm 0.04$ \\
 & SP $\downarrow$ ($10^{0}$) & $\mathbf{0.00 \pm 0.00}$ & $\mathbf{0.00 \pm 0.00}$ & $\mathbf{0.00 \pm 0.00}$ & $65.69 \pm 4.79$ & $0.27 \pm 0.27$ & $0.29 \pm 0.08$ \\
\cline{1-8}
\multirow[c]{3}{*}{DST-C} & EUM $\uparrow$ ($10^{1}$) & $0.33 \pm 0.99$ & $2.05 \pm 0.80$ & $0.08 \pm 0.08$ & $3.36 \pm 0.30$ & $-3.03 \pm 0.92$ & $\mathbf{4.15 \pm 0.48}$ \\
 & HV $\uparrow$ ($10^{3}$) & $2.01 \pm 1.28$ & $4.30 \pm 1.55$ & $0.36 \pm 0.26$ & $6.74 \pm 0.51$ & $0.18 \pm 0.15$ & $\mathbf{8.31 \pm 0.84}$ \\
 & SP $\downarrow$ ($10^{-1}$) & $56.45 \pm 45.70$ & $374.15 \pm 214.38$ & $\mathbf{7.61 \pm 7.61}$ & $12369.38 \pm 1051.80$ & $145.97 \pm 145.97$ & $1260.36 \pm 458.87$ \\
\cline{1-8}
\multirow[c]{3}{*}{FishWood} & EUM $\uparrow$ ($10^{1}$) & $9.16 \pm 0.02$ & $9.16 \pm 0.04$ & $7.93 \pm 0.10$ & $4.59 \pm 0.01$ & $\mathbf{9.22 \pm 0.02}$ & $7.65 \pm 0.48$ \\
 & HV $\uparrow$ ($10^{3}$) & $0.56 \pm 0.19$ & $0.35 \pm 0.14$ & $0.77 \pm 0.05$ & $0.05 \pm 0.03$ & $0.94 \pm 0.06$ & $\mathbf{1.56 \pm 0.05}$ \\
 & SP $\downarrow$ ($10^{1}$) & $134.80 \pm 60.33$ & $233.61 \pm 118.24$ & $\mathbf{4.08 \pm 1.06}$ & $176.73 \pm 16.74$ & $612.73 \pm 176.80$ & $36.26 \pm 6.93$ \\
\cline{1-8}
\multirow[c]{3}{*}{Fruit} & EUM $\uparrow$ ($10^{0}$) & $4.27 \pm 0.02$ & $4.30 \pm 0.02$ & $4.75 \pm 0.08$ & $4.41 \pm 0.08$ & $\mathbf{4.99 \pm 0.01}$ & $4.85 \pm 0.02$ \\
 & HV $\uparrow$ ($10^{4}$) & $0.53 \pm 0.01$ & $0.53 \pm 0.01$ & $1.42 \pm 0.13$ & $0.09 \pm 0.02$ & $\mathbf{1.55 \pm 0.01}$ & $1.52 \pm 0.02$ \\
 & SP $\downarrow$ ($10^{-4}$) & $7.07 \pm 0.34$ & $11.97 \pm 5.22$ & $6.78 \pm 0.70$ & $194325.24 \pm 10559.86$ & $\mathbf{3.19 \pm 0.39}$ & $4.05 \pm 0.44$ \\
\cline{1-8}
\multirow[c]{3}{*}{Lunar} & EUM $\uparrow$ ($10^{1}$) & $1.21 \pm 0.06$ & $\mathbf{1.24 \pm 0.05}$ & $1.15 \pm 0.05$ & $-0.44 \pm 0.63$ & $-0.11 \pm 0.15$ & $-0.86 \pm 0.08$ \\
 & HV $\uparrow$ ($10^{10}$) & $0.00 \pm 0.00$ & $0.47 \pm 0.30$ & $0.46 \pm 0.25$ & $2.91 \pm 0.90$ & $\mathbf{3.77 \pm 0.24}$ & $0.78 \pm 0.00$ \\
 & SP $\downarrow$ ($10^{1}$) & $\mathbf{1.38 \pm 1.38}$ & $172.98 \pm 124.98$ & $14.97 \pm 3.63$ & $13598.59 \pm 7974.63$ & $44.10 \pm 4.56$ & $18.57 \pm 2.91$ \\
\cline{1-8}
\multirow[c]{3}{*}{Minecart} & EUM $\uparrow$ ($10^{-1}$) & $-64.51 \pm 0.58$ & $-64.35 \pm 0.98$ & $\mathbf{-6.81 \pm 0.96}$ & $-9.90 \pm 8.21$ & $-61.68 \pm 2.15$ & $-41.07 \pm 4.67$ \\
 & HV $\uparrow$ ($10^{1}$) & $0.00 \pm 0.00$ & $0.00 \pm 0.00$ & $1.75 \pm 0.88$ & $0.00 \pm 0.00$ & $2.43 \pm 0.72$ & $\mathbf{6.51 \pm 0.69}$ \\
 & SP $\downarrow$ ($10^{0}$) & $\mathbf{0.00 \pm 0.00}$ & $\mathbf{0.00 \pm 0.00}$ & $0.86 \pm 0.37$ & $\mathbf{0.00 \pm 0.00}$ & $126.73 \pm 61.52$ & $8.50 \pm 3.63$ \\
\cline{1-8}
\multirow[c]{3}{*}{Reacher} & EUM $\uparrow$ ($10^{1}$) & $2.28 \pm 0.07$ & $2.29 \pm 0.07$ & $2.13 \pm 0.00$ & $2.29 \pm 0.02$ & $\mathbf{2.54 \pm 0.03}$ & $2.31 \pm 0.03$ \\
 & HV $\uparrow$ ($10^{6}$) & $1.36 \pm 0.19$ & $1.40 \pm 0.19$ & $1.06 \pm 0.01$ & $1.29 \pm 0.03$ & $\mathbf{2.08 \pm 0.07}$ & $1.47 \pm 0.08$ \\
 & SP $\downarrow$ ($10^{-2}$) & $4.28 \pm 1.97$ & $5.93 \pm 2.34$ & $\mathbf{2.00 \pm 0.50}$ & $9286.84 \pm 1410.67$ & $15.47 \pm 1.89$ & $5.77 \pm 1.61$ \\
\cline{1-8}
\multirow[c]{3}{*}{Room} & EUM $\uparrow$ ($10^{0}$) & $1.96 \pm 0.05$ & $1.96 \pm 0.06$ & $\mathbf{2.15 \pm 0.04}$ & $0.00 \pm 0.00$ & $1.39 \pm 0.09$ & $1.90 \pm 0.15$ \\
 & HV $\uparrow$ ($10^{1}$) & $0.74 \pm 0.08$ & $0.72 \pm 0.09$ & $\mathbf{1.02 \pm 0.06}$ & $0.00 \pm 0.00$ & $0.31 \pm 0.06$ & $0.82 \pm 0.14$ \\
 & SP $\downarrow$ ($10^{0}$) & $0.03 \pm 0.01$ & $0.03 \pm 0.01$ & $0.09 \pm 0.02$ & $\mathbf{0.00 \pm 0.00}$ & $0.13 \pm 0.03$ & $0.11 \pm 0.03$ \\
\cline{1-8}
\bottomrule
\end{tabular}
}
\end{table}

\subsubsection{Continuous action tasks}
\label{app:exp-detail-continuous}
We show, from Figure~\ref{fig:ec_hv_mf} to Figure~\ref{fig:ec_eum_mb},  for both of model-free and model-based methods, the sample efficiency curve with respect to the Hypervolume and EUM, and in Figure~\ref{fig:fronts_mf} and Figure~\ref{fig:fronts_mb}, final-step return scatters on 2-dimensional six tasks.
Additionally, we summarize qualitative metrics at Final step metrics in Table~\ref{tbl:full}.

\begin{figure}[tb]
    \centering
    \includegraphics[width=\linewidth]{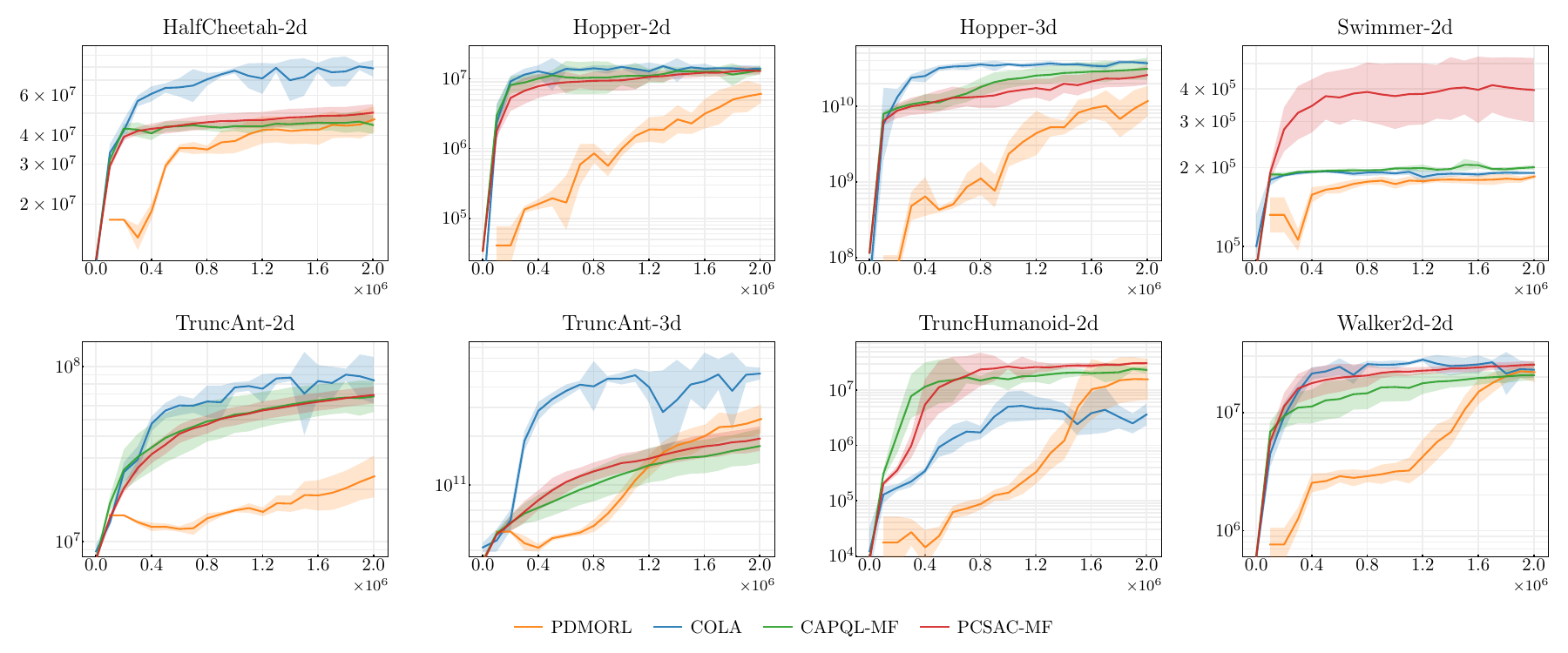}
    \caption{
      Hypervolume of model-free methods over envionment steps on eight continuous action tasks from MO-gymnasium.
      All statistics are obtained from 8 independent runs and the mean and standard deviation are reported.
      PD-MORL allocates an initial interaction budget to identify key preferences~\citep{basaklar2022pd}; hence its curve starts slightly later on the x-axis.
    }
    \label{fig:ec_hv_mf}
\end{figure}

\begin{figure}[tb]
    \centering
    \includegraphics[width=\linewidth]{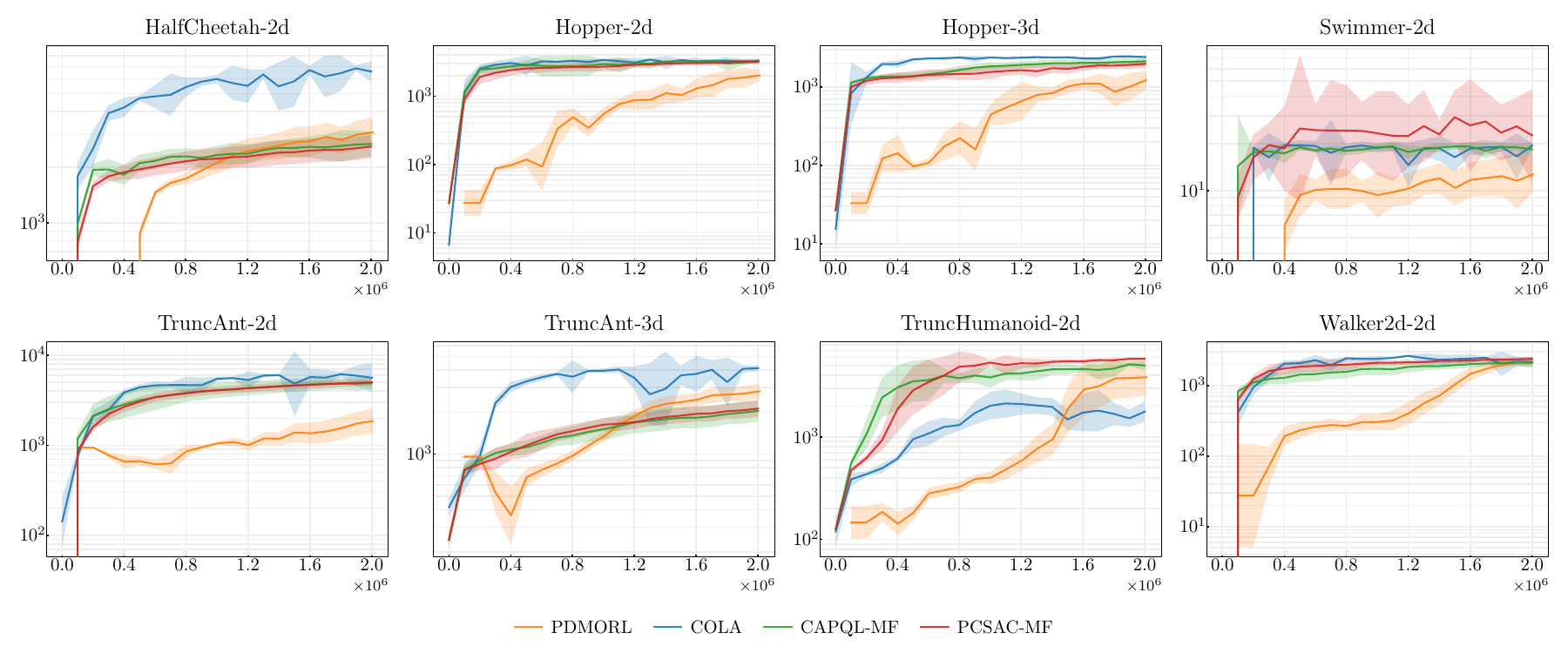}
    \caption{
      Expected Utility Metric of model-free methods over envionment steps on eight continuous action tasks from MO-gymnasium.
      All statistics are obtained from 8 independent runs and the mean and standard deviation are reported.
      PD-MORL allocates an initial interaction budget to identify key preferences~\citep{basaklar2022pd}; hence its curve starts slightly later on the x-axis.
    }
    \label{fig:ec_eum_mf}
\end{figure}

\begin{figure}[tb]
    \centering
    \includegraphics[width=\linewidth]{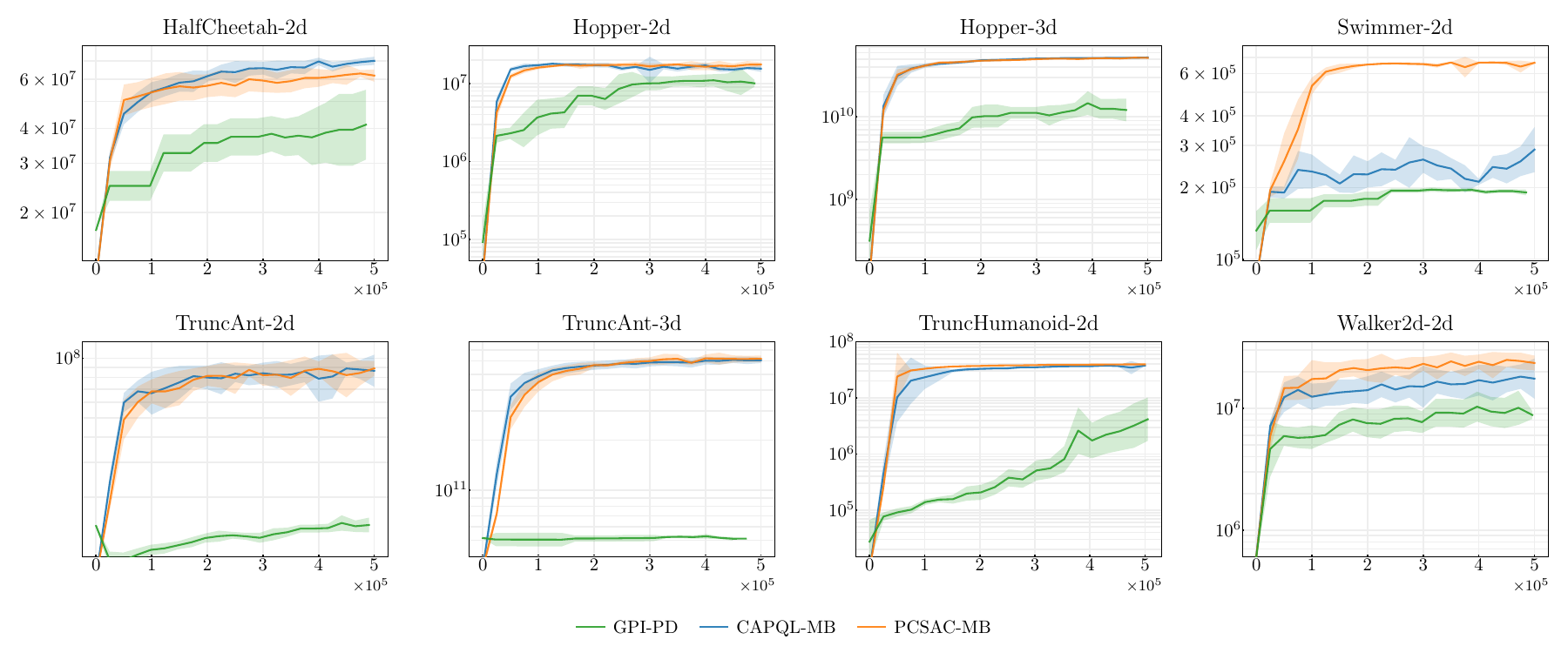}
    \caption{
      Hypervolume of model-based methods over envionment steps on eight continuous action tasks from MO-gymnasium.
      All statistics are obtained from 8 independent runs and the mean and standard deviation are reported.
    }
    \label{fig:ec_hv_mb}
\end{figure}

\begin{figure}[tb]
    \centering
    \includegraphics[width=\linewidth]{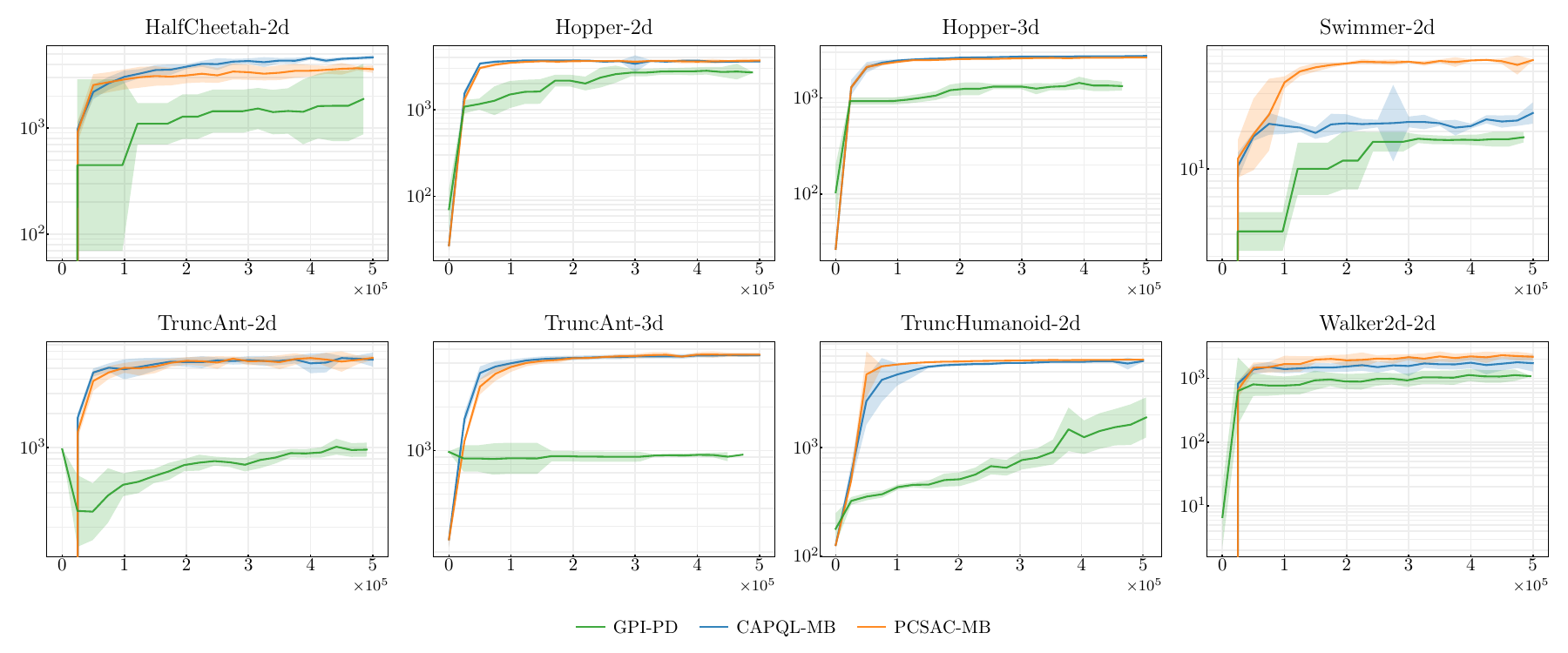}
    \caption{
      Expected Utility Metric of model-based methods over envionment steps on eight continuous action tasks from MO-gymnasium.
      All statistics are obtained from 8 independent runs and the mean and standard deviation are reported.
    }
    \label{fig:ec_eum_mb}
\end{figure}

\begin{figure}[tb]
    \centering
    \includegraphics[width=\linewidth]{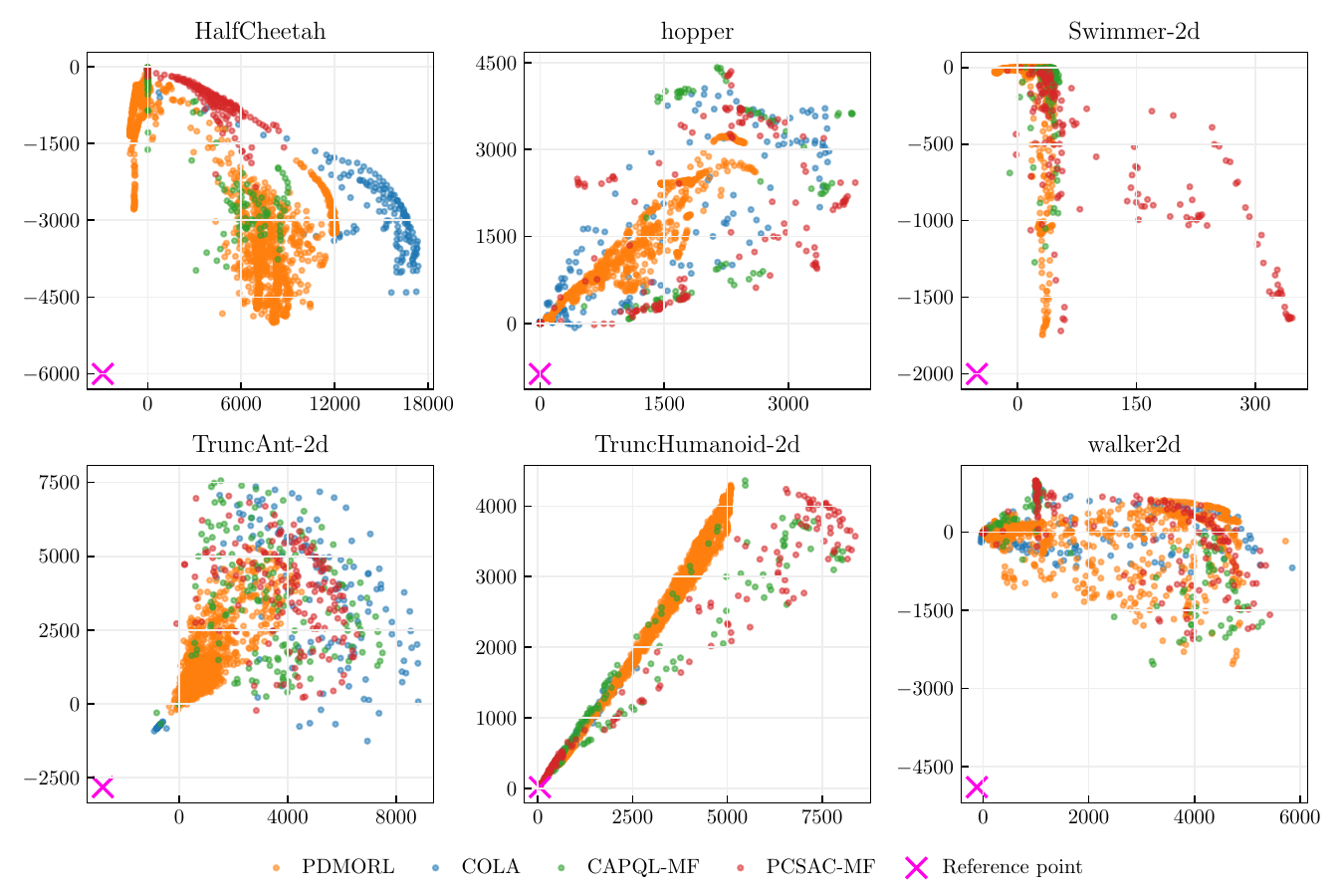}
    \caption{Final-step (2M steps) sum-of-reward-vectors scatter plots for model-free methods on 2-dimensional tasks.}
    \label{fig:fronts_mf}
\end{figure}

\begin{figure}[tb]
    \centering
    \includegraphics[width=\linewidth]{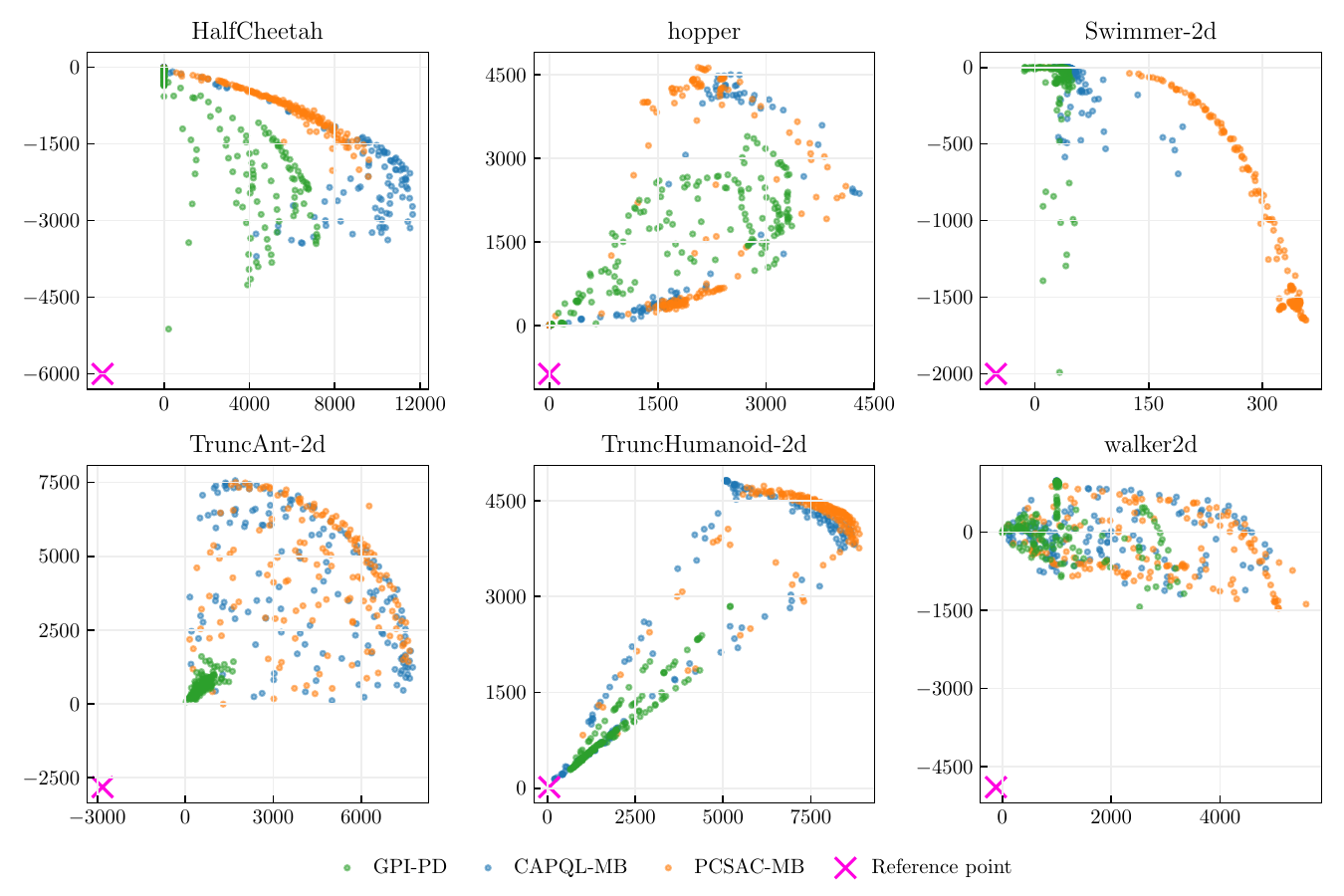}
    \caption{Final-step (500K steps) sum-of-reward-vectors scatter plots for model-based methods on 2-dimensional tasks.}
    \label{fig:fronts_mb}
\end{figure}

\begin{table}[tbh]
\caption{Comparison of final-step metrics, EUM, HV, and Sparsity on continuous acttion tasks from MO-gymnasium.}
\label{tbl:full}
\resizebox{\textwidth}{!}{%
\centering
\begin{tabular}{lc|ccccccc}
\toprule
 &  & PDMORL & COLA & CAPQL-MF & PCSAC-MF & GPI-PD & CAPQL-MB & PCSAC-MB \\
\midrule
\multirow[c]{3}{*}{HalfCheetah-2d} & EUM $\uparrow$ ($10^{3}$) & $3.39 \pm 0.32$ & $\mathbf{7.02 \pm 0.08}$ & $2.66 \pm 0.20$ & $2.57 \pm 0.20$ & $1.74 \pm 0.27$ & $4.64 \pm 0.10$ & $3.63 \pm 0.13$ \\
 & HV $\uparrow$ ($10^{7}$) & $5.43 \pm 0.33$ & $\mathbf{8.19 \pm 0.10}$ & $4.43 \pm 0.20$ & $5.03 \pm 0.23$ & $4.01 \pm 0.28$ & $7.00 \pm 0.12$ & $6.25 \pm 0.15$ \\
 & SP $\downarrow$ ($10^{6}$) & $0.59 \pm 0.04$ & $5.63 \pm 1.91$ & $5.48 \pm 1.27$ & $\mathbf{0.28 \pm 0.06}$ & $1.03 \pm 0.26$ & $4.91 \pm 1.06$ & $0.35 \pm 0.03$ \\
\cline{1-9}

\multirow[c]{3}{*}{Hopper-2d} & EUM $\uparrow$ ($10^{3}$) & $2.07 \pm 0.18$ & $2.99 \pm 0.35$ & $3.32 \pm 0.09$ & $3.15 \pm 0.10$ & $2.80 \pm 0.09$ & $3.59 \pm 0.06$ & $\mathbf{3.67 \pm 0.06}$ \\
 & HV $\uparrow$ ($10^{7}$) & $0.62 \pm 0.09$ & $1.27 \pm 0.18$ & $1.35 \pm 0.10$ & $1.29 \pm 0.07$ & $1.10 \pm 0.06$ & $1.56 \pm 0.08$ & $\mathbf{1.74 \pm 0.07}$ \\
 & SP $\downarrow$ ($10^{6}$) & $\mathbf{0.01 \pm 0.00}$ & $0.22 \pm 0.10$ & $1.13 \pm 0.48$ & $0.48 \pm 0.20$ & $0.20 \pm 0.11$ & $0.32 \pm 0.19$ & $1.00 \pm 0.19$ \\
\cline{1-9}

\multirow[c]{3}{*}{Hopper-3d} & EUM $\uparrow$ ($10^{3}$) & $1.26 \pm 0.20$ & $2.24 \pm 0.15$ & $2.09 \pm 0.12$ & $1.96 \pm 0.12$ & $1.31 \pm 0.10$ & $\mathbf{2.74 \pm 0.02}$ & $2.63 \pm 0.01$ \\
 & HV $\uparrow$ ($10^{10}$) & $1.27 \pm 0.29$ & $3.28 \pm 0.38$ & $3.08 \pm 0.32$ & $2.60 \pm 0.34$ & $1.21 \pm 0.20$ & $\mathbf{5.23 \pm 0.06}$ & $5.14 \pm 0.04$ \\
 & SP $\downarrow$ ($10^{6}$) & $0.03 \pm 0.01$ & $1.87 \pm 1.70$ & $0.05 \pm 0.01$ & $0.09 \pm 0.02$ & $0.02 \pm 0.01$ & $0.02 \pm 0.00$ & $\mathbf{0.01 \pm 0.00}$ \\
\cline{1-9}

\multirow[c]{3}{*}{Swimmer-2d} & EUM $\uparrow$ ($10^{1}$) & $1.40 \pm 0.19$ & $1.85 \pm 0.06$ & $1.88 \pm 0.07$ & $2.29 \pm 0.85$ & $1.76 \pm 0.06$ & $2.80 \pm 0.32$ & $\mathbf{7.51 \pm 0.05}$ \\
 & HV $\uparrow$ ($10^{5}$) & $1.89 \pm 0.02$ & $1.91 \pm 0.02$ & $2.01 \pm 0.02$ & $3.92 \pm 0.56$ & $1.92 \pm 0.02$ & $2.88 \pm 0.35$ & $\mathbf{6.67 \pm 0.04}$ \\
 & SP $\downarrow$ ($10^{4}$) & $0.14 \pm 0.06$ & $\mathbf{0.01 \pm 0.00}$ & $3.14 \pm 1.64$ & $5.25 \pm 2.14$ & $0.82 \pm 0.78$ & $0.34 \pm 0.12$ & $2.04 \pm 0.09$ \\
\cline{1-9}

\multirow[c]{3}{*}{TruncAnt-3d} & EUM $\uparrow$ ($10^{3}$) & $2.92 \pm 0.17$ & $4.21 \pm 0.06$ & $2.04 \pm 0.19$ & $2.13 \pm 0.17$ & $0.93 \pm 0.01$ & $4.53 \pm 0.05$ & $\mathbf{4.54 \pm 0.07}$ \\
 & HV $\uparrow$ ($10^{11}$) & $2.72 \pm 0.24$ & $5.00 \pm 0.14$ & $1.75 \pm 0.23$ & $1.96 \pm 0.22$ & $0.52 \pm 0.01$ & $6.06 \pm 0.12$ & $\mathbf{6.10 \pm 0.13}$ \\
 & SP $\downarrow$ ($10^{5}$) & $\mathbf{0.02 \pm 0.02}$ & $7.45 \pm 0.80$ & $0.57 \pm 0.11$ & $0.29 \pm 0.07$ & $5.34 \pm 2.27$ & $0.06 \pm 0.01$ & $0.15 \pm 0.05$ \\
\cline{1-9}

\multirow[c]{3}{*}{TruncAnt-2d} & EUM $\uparrow$ ($10^{3}$) & $2.11 \pm 0.42$ & $5.42 \pm 0.88$ & $4.94 \pm 0.36$ & $5.06 \pm 0.21$ & $1.03 \pm 0.09$ & $5.93 \pm 0.34$ & $\mathbf{6.17 \pm 0.15}$ \\
 & HV $\uparrow$ ($10^{7}$) & $2.67 \pm 0.51$ & $8.14 \pm 1.15$ & $6.78 \pm 0.68$ & $6.96 \pm 0.38$ & $1.50 \pm 0.07$ & $8.64 \pm 0.63$ & $\mathbf{8.97 \pm 0.35}$ \\
 & SP $\downarrow$ ($10^{6}$) & $0.34 \pm 0.09$ & $0.86 \pm 0.21$ & $0.79 \pm 0.16$ & $0.44 \pm 0.06$ & $\mathbf{0.14 \pm 0.09}$ & $2.10 \pm 1.04$ & $3.61 \pm 2.97$ \\
\cline{1-9}

\multirow[c]{3}{*}{TruncHumanoid-2d} & EUM $\uparrow$ ($10^{3}$) & $4.00 \pm 0.44$ & $1.70 \pm 0.34$ & $4.93 \pm 0.46$ & $5.85 \pm 0.05$ & $2.17 \pm 0.38$ & $6.28 \pm 0.04$ & $\mathbf{6.46 \pm 0.02}$ \\
 & HV $\uparrow$ ($10^{7}$) & $1.69 \pm 0.24$ & $0.35 \pm 0.13$ & $2.32 \pm 0.30$ & $3.09 \pm 0.06$ & $0.50 \pm 0.16$ & $3.78 \pm 0.04$ & $\mathbf{3.95 \pm 0.03}$ \\
 & SP $\downarrow$ ($10^{5}$) & $0.00 \pm 0.00$ & $\mathbf{0.00 \pm 0.00}$ & $1.42 \pm 0.78$ & $0.82 \pm 0.20$ & $\mathbf{0.00 \pm 0.00}$ & $2.76 \pm 0.44$ & $0.52 \pm 0.08$ \\
\cline{1-9}

\multirow[c]{3}{*}{Walker2d-2d} & EUM $\uparrow$ ($10^{3}$) & $2.31 \pm 0.14$ & $2.37 \pm 0.13$ & $2.11 \pm 0.09$ & $\mathbf{2.39 \pm 0.08}$ & $1.12 \pm 0.09$ & $1.73 \pm 0.24$ & $2.08 \pm 0.18$ \\
 & HV $\uparrow$ ($10^{7}$) & $2.41 \pm 0.14$ & $2.47 \pm 0.13$ & $2.08 \pm 0.09$ & $\mathbf{2.54 \pm 0.09}$ & $1.01 \pm 0.15$ & $1.75 \pm 0.29$ & $2.23 \pm 0.22$ \\
 & SP $\downarrow$ ($10^{6}$) & $\mathbf{0.25 \pm 0.19}$ & $1.49 \pm 0.64$ & $1.75 \pm 0.29$ & $0.99 \pm 0.38$ & $0.25 \pm 0.12$ & $0.42 \pm 0.10$ & $0.58 \pm 0.20$ \\
\cline{1-9}
\bottomrule
\end{tabular}
} %
\end{table}

\subsection{Sensitivity analysis}
\label{app:sensitivity}
\subsubsection{Utopia point estimation}
\label{app:ur}
Since utopia-point estimation is one of the main practical gaps in the deep setting, we ran an additional sensitivity study by replacing the estimated utopia point $\hat{I}$ with $\lambda\hat{I}$ on 8 discrete-action MO-Gymnasium tasks.
Across the tested range as shonw in Table~\ref{tbl:ur_cmdpi} and Table~\ref{tbl:ur_pcsac}, neither CMDPI nor PCSAC showed a consistently strong degradation pattern across HV / EUM / SP.
The best $\lambda$ varied somewhat by metric, but we did not observe extreme brittleness.

\begin{table}
\centering
\caption{Senstivity of utopia point estimation bias for CMDPI}
\label{tbl:ur_cmdpi}
\begin{tabular}{llllll}
\toprule
value & 0.1 & 0.5 & 1.0 & 2.0 & 5.0 \\
\midrule
EUM & 3.4 $\pm$ 0.34 & 2.5 $\pm$ 0.42 & 3 $\pm$ 0.43 & 2.8 $\pm$ 0.3 & 3.4 $\pm$ 0.31 \\
HV & 3.2 $\pm$ 0.48 & 3.4 $\pm$ 0.3 & 2.9 $\pm$ 0.43 & 2.2 $\pm$ 0.49 & 3.2 $\pm$ 0.31 \\
SP & 2.8 $\pm$ 0.47 & 3.4 $\pm$ 0.48 & 2.2 $\pm$ 0.3 & 4 $\pm$ 0.35 & 2.6 $\pm$ 0.38 \\
\bottomrule
\end{tabular}
\end{table}

\begin{table}
\centering
\caption{Senstivity of utopia point estimation bias for PCSAC}
\label{tbl:ur_pcsac}
\begin{tabular}{llllll}
\toprule
value & 0.1 & 0.5 & 1.0 & 2.0 & 5.0 \\
\midrule
EUM & 3.9 $\pm$ 0.45 & 3.6 $\pm$ 0.55 & 2.4 $\pm$ 0.27 & 2.4 $\pm$ 0.34 & 2.8 $\pm$ 0.54 \\
HV & 2.5 $\pm$ 0.39 & 3.2 $\pm$ 0.68 & 2.6 $\pm$ 0.27 & 2.9 $\pm$ 0.4 & 3.7 $\pm$ 0.57 \\
SP & 2.6 $\pm$ 0.57 & 2.1 $\pm$ 0.39 & 2.8 $\pm$ 0.59 & 3.4 $\pm$ 0.35 & 4.1 $\pm$ 0.37 \\
\bottomrule
\end{tabular}
\end{table}

\subsubsection{Relative smoothness constant}
In practice, temperature parameter originating from the relative smoothness constant is treated as a hyperparameter. To address this concern empirically, in the similar way as experiments in~\ref{app:ur}, we swept the temperature parameter for both of CMDPI and PCSAC (here, denotated as $\alpha_{\text{ent}}$) and again found no consistently strong sensitivity across the tested range (Table~\ref{tbl:rs_cmdpi} and Table~\ref{tbl:rs_pcsac}).

\begin{table}
\centering
\caption{Sensitivity of the temperature $\alpha$ for CMDPI}
\label{tbl:rs_cmdpi}
\begin{tabular}{llllll}
\toprule
value & 0.01 & 0.1 & 1.0 & 10.0 & 100.0 \\
\midrule
EUM & 3.6 $\pm$ 0.39 & 2.3 $\pm$ 0.52 & 2.1 $\pm$ 0.25 & 3 $\pm$ 0.42 & 4 $\pm$ 0.43 \\
HV & 4.2 $\pm$ 0.35 & 2.6 $\pm$ 0.51 & 1.9 $\pm$ 0.38 & 2.9 $\pm$ 0.33 & 3.5 $\pm$ 0.38 \\
SP & 2.6 $\pm$ 0.54 & 2.4 $\pm$ 0.46 & 3.1 $\pm$ 0.4 & 3.2 $\pm$ 0.23 & 3.6 $\pm$ 0.49 \\
\bottomrule
\end{tabular}

\end{table}

\begin{table}
\centering
\caption{Sensitivity of the temperature $\alpha_{\text{ent}}$ for PCSAC}
\label{tbl:rs_pcsac}
\begin{tabular}{llllll}
\toprule
value & 0.01 & 0.1 & 1.0 & 10.0 & 100.0 \\
\midrule
EUM & 2.7 $\pm$ 0.45 & 2.2 $\pm$ 0.47 & 2.4 $\pm$ 0.3 & 4.4 $\pm$ 0.33 & 3.3 $\pm$ 0.47 \\
HV & 3.2 $\pm$ 0.53 & 2.4 $\pm$ 0.54 & 2.6 $\pm$ 0.54 & 3.4 $\pm$ 0.45 & 3.4 $\pm$ 0.38 \\
SP & 2.9 $\pm$ 0.47 & 2.9 $\pm$ 0.34 & 2.3 $\pm$ 0.52 & 3.6 $\pm$ 0.29 & 3.2 $\pm$ 0.38 \\
\bottomrule
\end{tabular}
\end{table}

\subsection{Compute Resources}

Experiments were conducted on an internal compute cluster;
a typical node has an Intel Xeon Gold 6138 CPU, an NVIDIA GeForce RTX 2080 Ti GPU, and about 360 GB RAM.
Each continuous-control run uses 1 CPU core, 1 GPU, and approximately 2 GB of main memory.

\paragraph{Continuous-Action Tasks}

For continuous control tasks (8 tasks), each (task, seed) run takes approximately 24 hours.
With 8 seeds, this corresponds to about 192 GPU hours per task and 1,536 GPU hours per method across all tasks.
Across 7 methods, this yields approximately 10,752 GPU hours.
Including implementation iterations and hyperparameter tuning (approximately 2--3 rounds),
the total cost is roughly 20,000--30,000 GPU hours.

In practice, we parallelize multiple seeds on a single GPU when possible,
reducing the effective wall-clock time to approximately one-third to one-half of the total GPU hours,
i.e., roughly 7,000--15,000 hours.

\paragraph{Discrete-Action Tasks}

For discrete tasks (8 tasks), experiments are conducted on CPU only, with each (task, seed) run taking approximately 30 minutes.
With 8 seeds, this corresponds to about 4 CPU hours per task and 32 CPU hours per method.
Across 6 methods, this yields approximately 192 CPU hours.
Including implementation iterations and hyperparameter tuning (approximately 2--3 rounds),
the total cost is approximately 400--600 CPU hours,
which remains negligible compared to continuous-control experiments.

\section{Limitations}

Our theoretical analysis is limited to the tabular setting, 
and bridging the gap between theory and deep implementations remains an important direction for future work.

(i) The assumptions of a tabular MOMDP and full support of the initial state distribution are introduced to ensure that CMDPI covers all points on a convex Pareto front. 
These assumptions are primarily technical and do not preclude applicability beyond this setting. 
In particular, since STCH is known to recover Pareto-optimal points even on non-convex fronts, 
it is an important open question to characterize the weakest conditions under which preference-to-solution uniqueness 
and convergence guarantees can be extended beyond the tabular and full-support assumptions.

(ii) While our theory guarantees a single preference-conditioned policy in the tabular case, 
extending this property to practical MORL settings with function approximation may introduce failure modes. 
A systematic theoretical and empirical analysis of such failure modes is left for future work.

(iii) Our analysis relies on smooth and strictly concave utility functions. 
We do not investigate, either theoretically or empirically, 
the behavior of CMDPI and STCH under other classes of utility functions (e.g., quadratic utilities), 
which may exhibit different optimization characteristics.

\end{document}